%% file: main.tex
\documentclass[conference]{IEEEtran}
\usepackage{times}

\usepackage[numbers]{natbib}
\usepackage{multicol}
\usepackage[bookmarks=true]{hyperref}

\usepackage{macros}

\begin{document}

\title{Informative Path Planning\\with Guaranteed Estimation Uncertainty}

\author{\authorblockN{Kalvik Jakkala\authorrefmark{1}, Saurav Agarwal\authorrefmark{2}, Jason O'Kane\authorrefmark{1} and Srinivas Akella\authorrefmark{3}}
\authorblockA{\authorrefmark{1}Texas A\&M University, \authorrefmark{2}Indian Institute of Technology Bombay, \authorrefmark{3}University of North Carolina at Charlotte}
\authorblockA{Email: \{kalvik, jokane\}@tamu.edu, saurav.agarwal@iitb.ac.in, sakella@charlotte.edu}}

\maketitle

\input{abstract}

\IEEEpeerreviewmaketitle

\input{introduction}
\input{related_work}
\input{preliminaries}
\input{problem_formulation}
\input{method}

\input{experiments}
\input{conclusion}

\clearpage
\section*{Acknowledgements}
We thank Pratap Tokekar and Shamak Dutta for sharing code for the HexCover method, and Weizhe Chen for assistance with the attentive kernel code.
We are grateful to Artur Wolek for providing access to the ASV platform, and we thank Ninh Nguyen for support during the ASV field trial.
We also thank Bennett Carley and Jeremy Mallette for assistance with Aqua2 setup, and Elias Sokolova, Jalil Francisco Chavez Galaviz, and Chelsey Edge for their help deploying the Aqua2 AUV in the ocean.
Finally, we thank Edwin Meriaux and Shuo Wen for capturing aerial footage of the AUV experiment.

\bibliography{references}

\input{appendix}


\end{document}

%% file: abstract.tex
\begin{abstract}
Environmental monitoring robots often need to estimate data fields (e.g., salinity, temperature, bathymetry) under tight resource constraints. Classical boustrophedon lawnmower surveys provide geometric coverage guarantees but can waste effort by oversampling predictable regions. In contrast, informative path planning~(IPP) methods leverage spatial correlations to reduce oversampling, yet typically offer no guarantees on estimation quality. This paper bridges these approaches by addressing IPP with guaranteed estimation uncertainty in complex environments: computing the shortest path whose measurements ensure that the Gaussian process~(GP) posterior variance---an intrinsic uncertainty measure that lower-bounds the mean-squared prediction error under the GP model---is upper bounded by a user-specified threshold over the monitoring region.

We propose a three-stage approach for efficient environmental monitoring: (i) learning a GP model from prior information; (ii) transforming the GP kernel into binary coverage maps that identify locations where uncertainty can be reduced below a target threshold; and (iii) planning a near-shortest route to satisfy the global uncertainty constraint. Our approach incorporates non-stationary kernels to capture spatially varying correlations in heterogeneous phenomena and accommodates non-convex environments with obstacles. We provide near-optimal approximation guarantees for both sensing-location selection and the joint selection-and-routing problem under a travel budget. Experiments on real-world topographic data demonstrate that our planners achieve uncertainty targets with fewer sensing locations and shorter travel distances than representative baselines. Furthermore, field experiments with autonomous surface and underwater vehicles validate the real-world feasibility of the approach. Our code is available at: \href{www.sgp-tools.com}{www.sgp-tools.com}
\end{abstract}

%% file: introduction.tex
\section{Introduction}

\begin{figure}[!ht]
   \centering
    \begin{subfigure}{0.24\textwidth}
        \includegraphics[width=\textwidth]{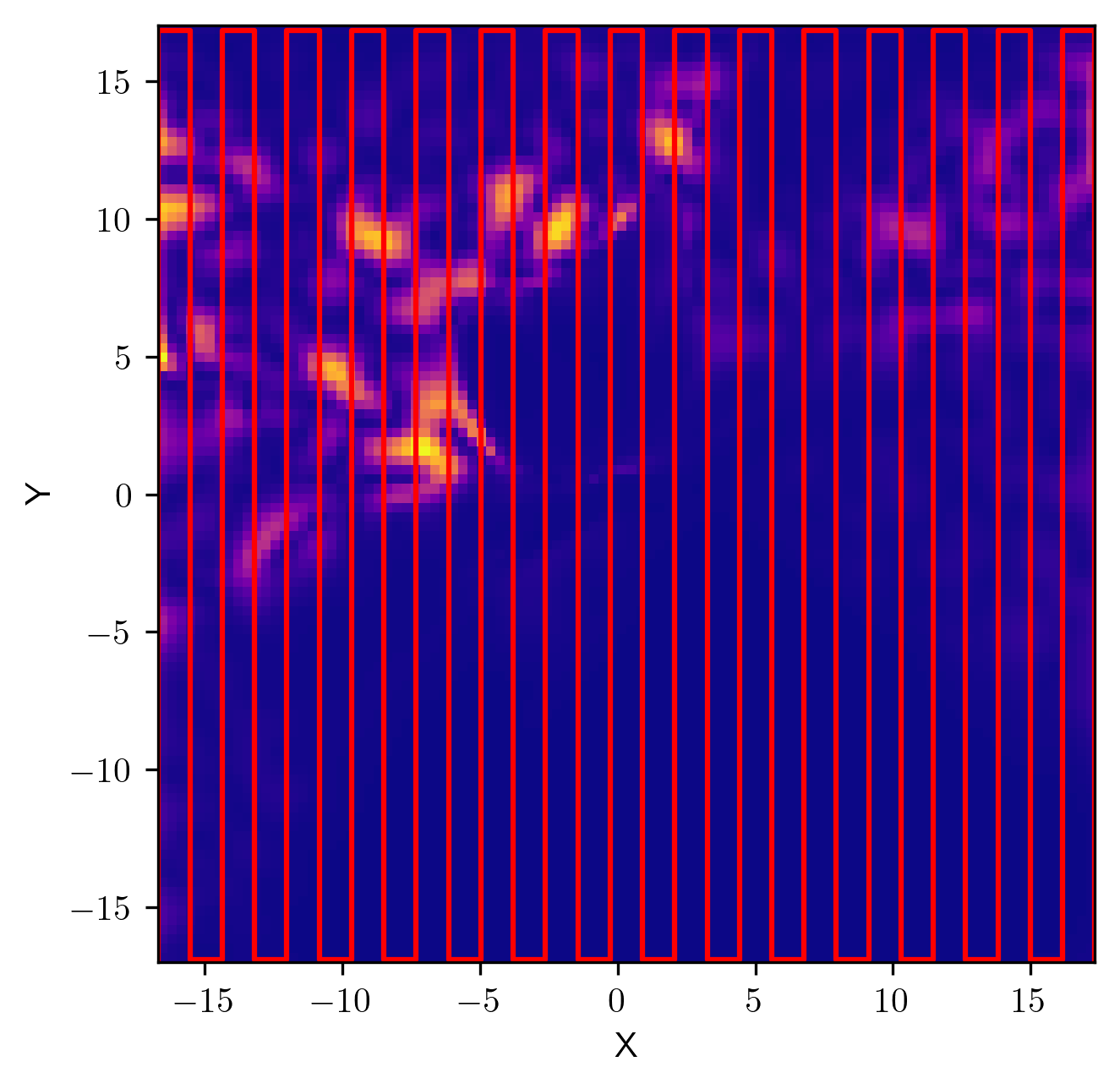}
        \caption{\textsc{Lawnmower}~\cite{ChosetP98}\\Path Length: 1047 m\\Coverage Guarantee}
    \end{subfigure}
    \hfill
    \begin{subfigure}{0.24\textwidth}
        \includegraphics[width=\textwidth]{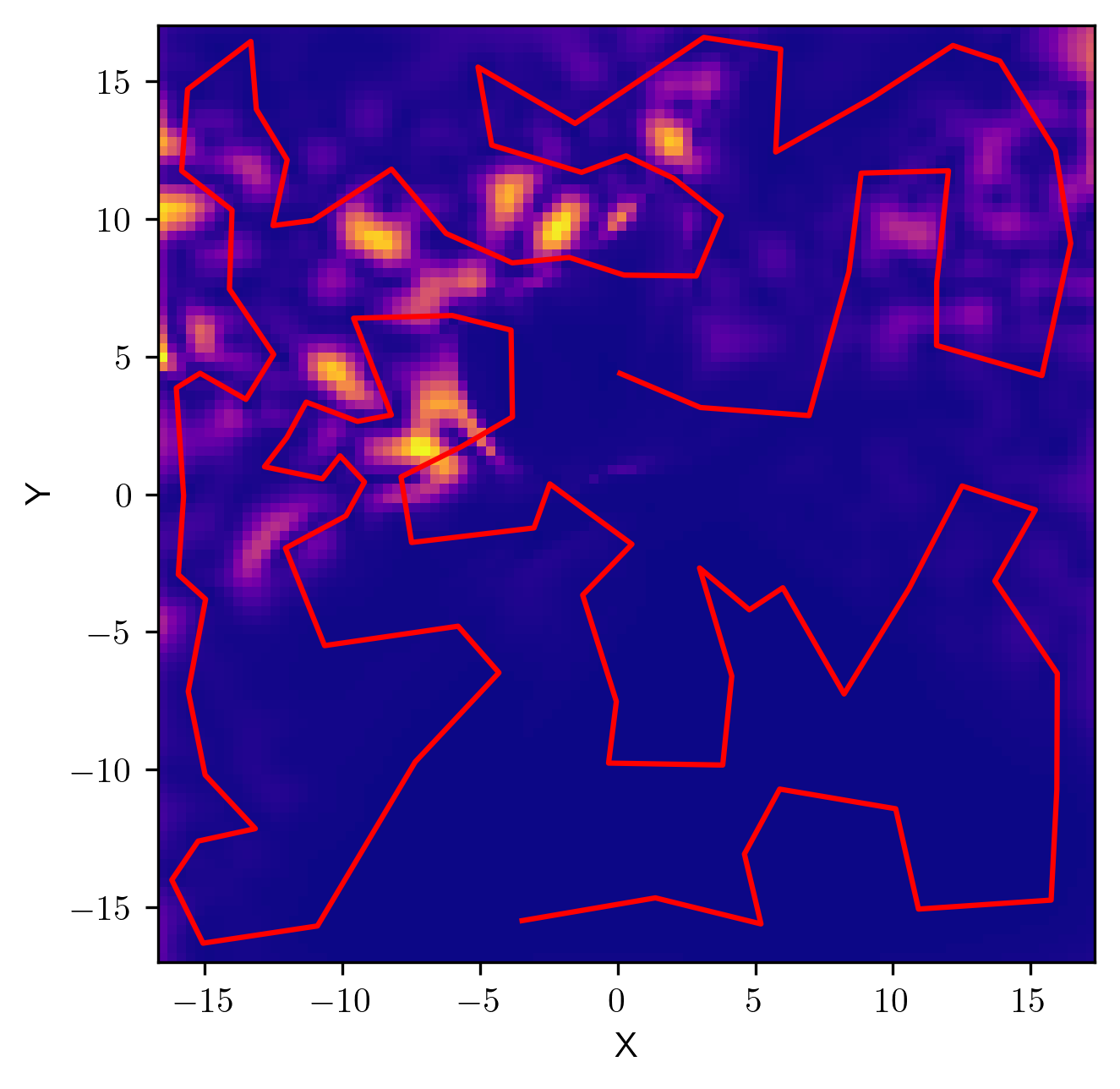}
        \caption{\textsc{Continuous-SGP}~\cite{JakkalaA25b}\\Path Length: 300 m\\Asymptotic Guarantee}
    \end{subfigure}
    \hfill
    \begin{subfigure}{0.24\textwidth}
        \includegraphics[width=\textwidth]{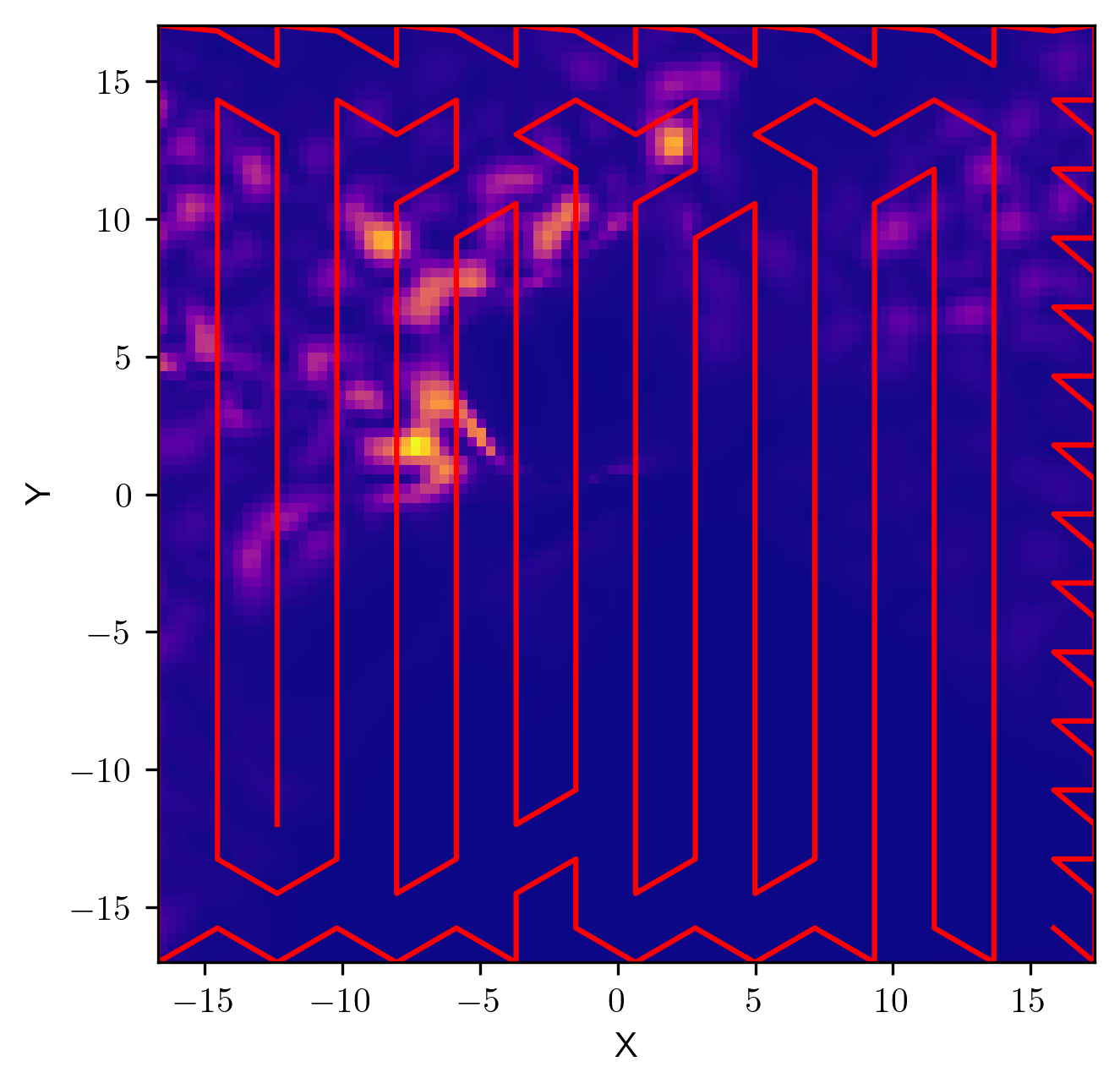}
        \caption{\textsc{HexCover}~\cite{DuttaWTS23}\\Path Length: 614 m\\Uncertainty Guarantee}
    \end{subfigure}
    \hfill
    \begin{subfigure}{0.24\textwidth}
        \includegraphics[width=\textwidth]{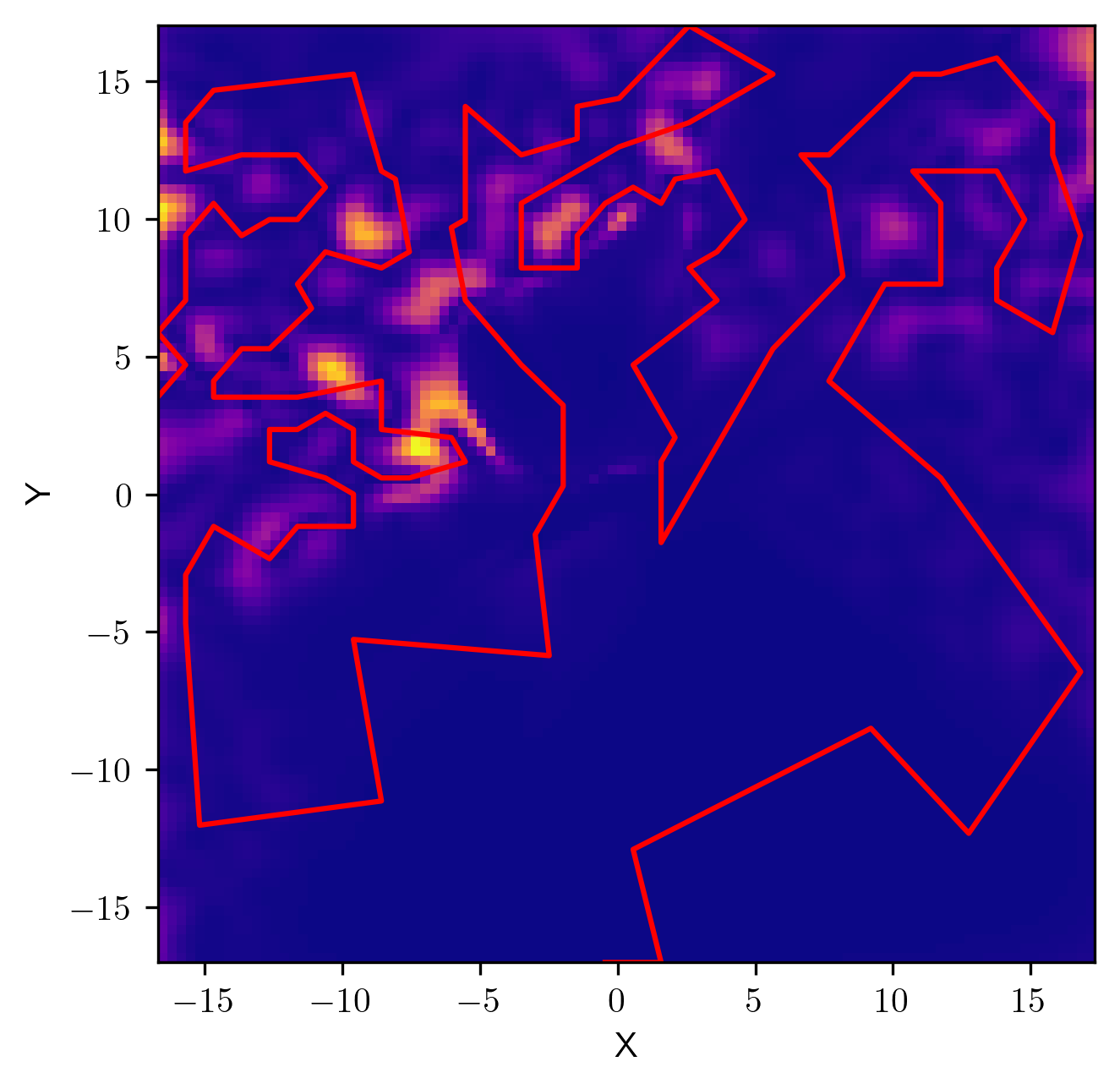}
        \caption{\textsc{GreedyCover} (Ours)\\Path Length: 238 m\\Uncertainty Guarantee}
    \end{subfigure}
    \caption{Comparison of IPP solution paths (red) and the resulting GP predictive uncertainty (heatmaps). While all methods maintain an identical mean predictive uncertainty of 0.02, the proposed \textsc{GreedyCover} approach yields the shortest path while guaranteeing that uncertainty remains below the target threshold across all evaluation locations.}    \label{fig:intro}
\end{figure}

Monitoring environmental variables such as ocean currents, salinity, and bathymetry is a critical task of broad scientific and societal importance~\cite{LovettBDJMRSLH07}. These quantities are both sensitive indicators of coupled Earth-system dynamics and active drivers of the climate, natural hazards, ecosystems, and resource availability. Accordingly, systematic observation strategies that provide robustness through model-based guarantees, efficiency through resource-aware optimization, and repeatability through automation are essential for effective risk reduction and sustainable management.

Despite its importance, environmental monitoring under finite resource constraints presents significant challenges. A common approach is the boustrophedon \emph{lawnmower} pattern, which provides strong geometric coverage guarantees by sweeping the sensing footprint across the entire region~\cite{ChosetP98,KarapetyanMLLOR18,AgarwalA22}. However, many environmental fields exhibit strong spatial correlations. In these contexts, exhaustive geometric coverage often results in wasted effort, oversampling regions that are easily predictable from nearby observations.

Informative path planning~(IPP) leverages spatial correlations---often through Gaussian processes~(GPs)~\cite{RasmussenW05}---to reduce travel while maintaining reconstruction quality. However, most IPP methods optimize surrogate measures of information gain without certifying the quality of the resulting reconstruction. Moreover, existing approaches typically adopt stopping criteria based on resource budgets, such as mission time, travel distance, or a cardinality constraint on sensing locations. Real-world deployments require a more principled approach: a stopping criterion tied directly to the desired reconstruction accuracy.

Recent research has begun to bridge this gap by introducing \textit{estimation uncertainty guarantees}. These methods provide an upper bound on the GP predictive variance, which in turn lower-bounds the mean-squared error~(MSE)~\cite{SuryanT20,DuttaWTS23}. While predictive uncertainty is only a proxy for accuracy, it serves as a rigorous measure of the information gained relative to the prior data, providing a quantitative signal of mission progress. However, existing methods with such guarantees are limited in two critical ways: they generally rely on Radial Basis Function (RBF) kernels—forcing a stationary correlation model—and are derived only for simplified, convex environments. Such constraints are rarely met in the real-world, where environmental fields can be highly non-stationary and operational areas often involve obstacles or complex, non-convex boundaries.

\begin{figure}[!t]
   \centering
    \begin{subfigure}{0.24\textwidth}
        \includegraphics[width=\textwidth]{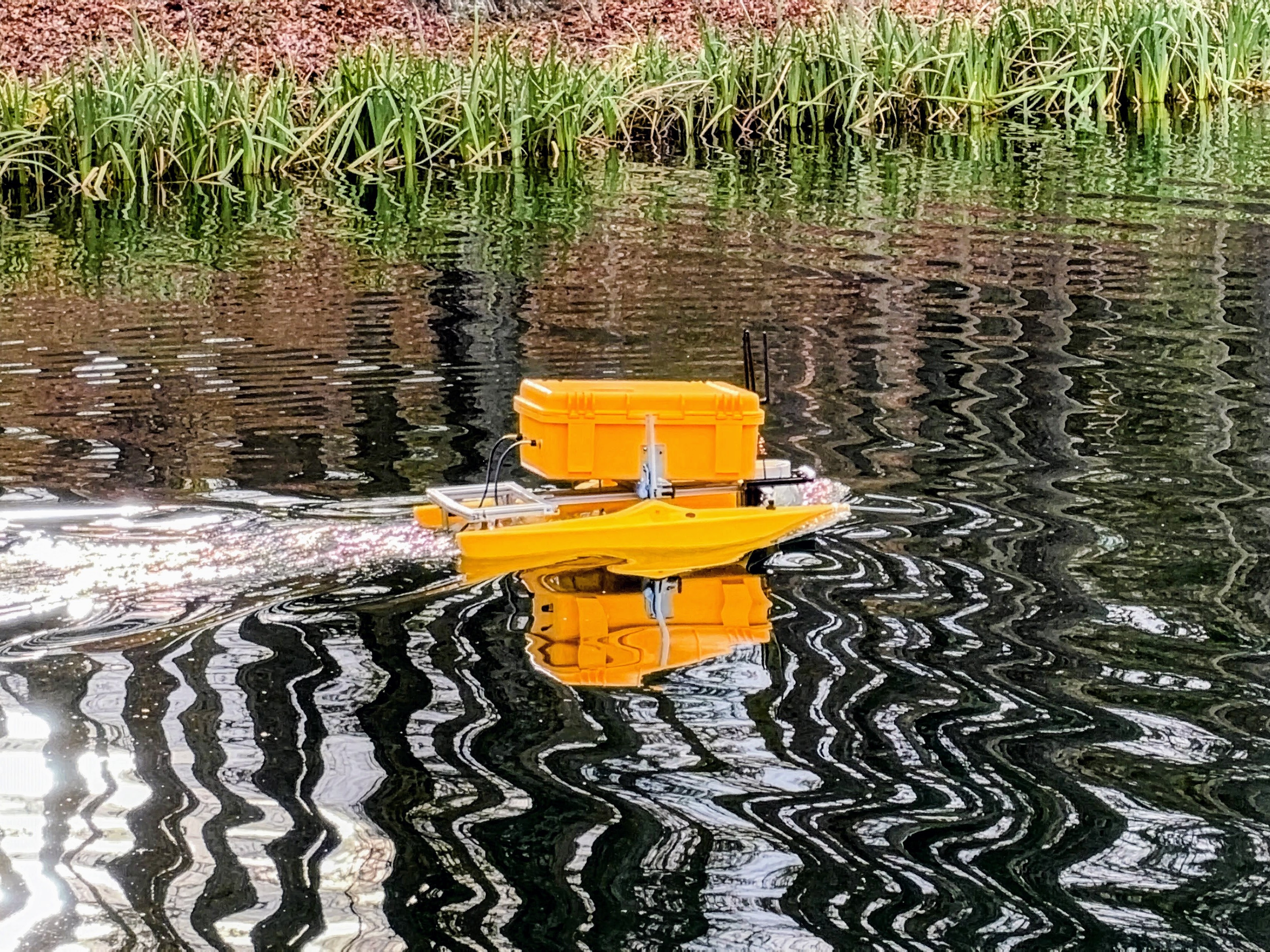}
    \end{subfigure}
    \hfill
    \begin{subfigure}{0.24\textwidth}
        \includegraphics[width=\textwidth]{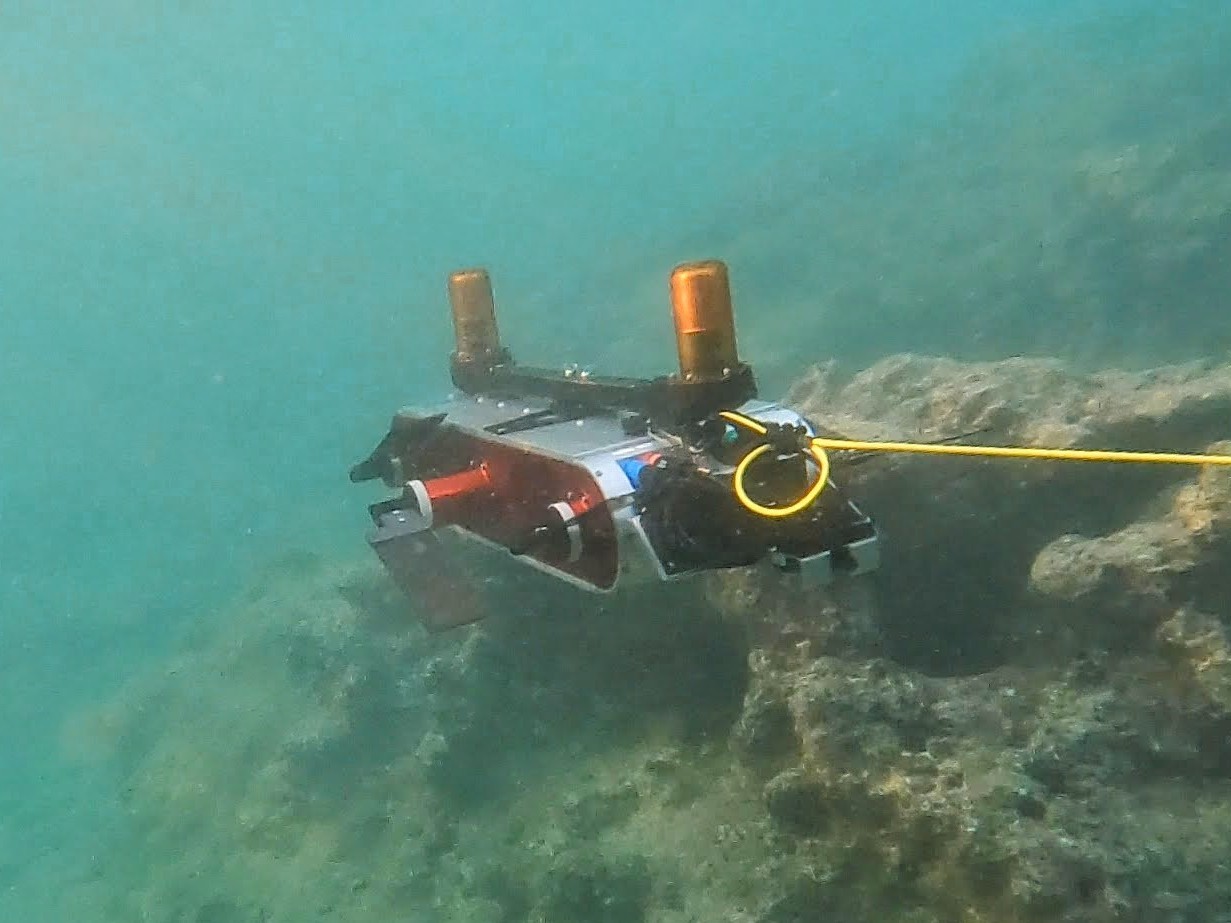}
    \end{subfigure}
    \caption{An autonomous surface vehicle~(ASV) and an autonomous underwater vehicle~(AUV) mapping bathymetry during our field trials.}
    \label{fig:cover}
\end{figure}

This paper addresses these limitations by developing new methods for IPP with estimation uncertainty guarantees in complex environments. Figure~\ref{fig:intro} highlights the resulting efficiency gains in a non-stationary bathymetry setting relative to representative baselines. Our main contributions are:
\begin{itemize}
    \item \textbf{IPP with Uncertainty Guarantees:} We present an informative path planning approach that upper-bounds the GP predictive variance, ensuring the maximum predictive uncertainty remains below a user-specified threshold.
    
    \item \textbf{Non-Stationary and Non-Convex Modeling:} We introduce a method that captures and leverages non-stationary correlations while accommodating non-convex environments containing obstacles and complex boundaries.

    \item \textbf{Near-Optimal Algorithms:} We provide near-optimal approximation algorithms for both sensing-location selection and the joint location-selection and routing problem.

    \item \textbf{Field Validation:} We validate our approach through field trials using autonomous surface and underwater vehicles for bathymetry mapping, demonstrating the practical feasibility of our method in real-world conditions.
\end{itemize}

%% file: related_work.tex
\section{Related Work}
\label{sec:related_work}

\textbf{Lawnmower Planners:}
A foundational approach in environmental monitoring is the boustrophedon (lawnmower) sweep, where a robot follows parallel transects to provide deterministic coverage based on its sensing footprint~\cite{ChosetP98}. This paradigm has been adapted to account for realistic vehicle dynamics, such as Dubins kinematics constraints for autonomous surface vehicles~\cite{KarapetyanMLLOR18}, and extended to multi-agent systems where area coverage objectives must be balanced against practical constraints like limited onboard energy or communication range~\cite{AgarwalA22}.

\textit{Contrast to our work:}
Lawnmower planners guarantee geometric visitation but oversample predictable areas by ignoring data correlations. Our approach instead optimizes for a worst-case posterior uncertainty threshold, enabling non-sweep paths that skip redundant regions while guaranteeing a global uncertainty bound.

\textbf{Coverage Planners:}
Our work is also closely related to viewpoint and sensor placement formulations that offer geometric alternatives to exhaustive sweeping. These include randomized art-gallery algorithms designed to select discrete sensor locations that collectively \emph{see} an entire environment~\cite{BanosL01}. More broadly, the field encompasses decentralized and adaptive strategies targeting geometric visitation or sensing density, particularly in complex or multi-robot settings. For instance, Voronoi-based methods provide coverage guarantees in non-convex environments for networked teams~\cite{BreitenmoserSMSR10}, while robust adaptive control laws address uncertainty in the underlying sensor network objectives~\cite{SchwagerVPRT17}. In unknown environments, recent methods trade off exploration (mapping free space) with coverage progress to improve operational efficiency~\cite{BoumanOKCKLAB22}.

\textit{Contrast to our work:}
Unlike coverage planners that require full-field traversal for geometric visibility, we optimize a \emph{coverage-as-uncertainty} objective. By imposing a global worst-case posterior variance threshold, we certify estimation uncertainty while purposefully skipping regions that do not require direct sensing.

\textbf{Informative Path Planning Methods:}
Informative Path Planning~(IPP) and Robotic Information Gathering~(RIG) focus on planning paths that maximize information gain regarding an unknown field, frequently employing Gaussian Processes~(GPs)~\cite{RasmussenW05} to model spatial correlations. Submodular formulations have provided a rigorous foundation for this, enabling efficient approximation algorithms for sensor placement~\cite{krauseSG08}, including extensions that incorporate travel and routing constraints~\cite{ChekuriP05, BottarelliBBF19}.

In robotic monitoring, IPP has been extensively studied for waypoint optimization in spatiotemporal settings~\cite{BinneyKS13}, adaptive continuous-space planning for environmental monitoring~\cite{HitzGGPS17}, and Bayesian exploration for mapping~\cite{FrancisOMR19}. Recent research has expanded into sequential optimization-based approaches~\cite{OttKB24} and distributed mapping for heterogeneous teams using wavelet-based ergodic exploration~\cite{RaoCW25}. Parallel to these methods, learning-based techniques—such as deep reinforcement learning for adaptive planning~\cite{RuckinJP23} and feature tracking~\cite{FuentesPNB25}—have gained traction. Other efforts focus on model expressiveness and scalability, utilizing multi-scale sampling~\cite{ManjannaD17}, attentive kernels for information gathering~\cite{ChenKL22, ChenLK24}, and differentiable optimization via sparse Gaussian Processes~\cite{JakkalaA24, JakkalaA25, JakkalaA25b}. Furthermore, routing-constrained IPP has been applied to high-stakes estimation tasks, such as localized gas leak rate quantification~\cite{JakkalaA22}.

\textit{Contrast to our work:}
Unlike most IPP approaches that maximize information gain without providing formal guarantees on estimation quality, we plan with respect to an explicit global uncertainty constraint. By targeting the worst-case posterior variance, our method ensures that the resulting paths satisfy prescribed estimation uncertainty requirements.

\textbf{IPP Methods with Estimation Uncertainty Guarantees:}
The most relevant approaches to our work belong to the specialized class of IPP methods that provide formal \emph{estimation uncertainty guarantees}. Unlike standard information-gathering techniques, these methods certify that a reconstructed field satisfies a user-specified uncertainty threshold. For instance, Suryan and Tokekar investigated minimum-time field learning with multi-robot teams, leveraging GP structures to derive performance guarantees under operational time constraints~\cite{SuryanT20}. Building on this foundation, Dutta~\emph{et al.} developed approximation algorithms for the dual sensor placement and shortest-path problems, strengthening theoretical guarantees within convex environments~\cite{DuttaWTS23}. Collectively, these works demonstrate that rigorous, coverage-style certification can be successfully integrated into the informative sampling paradigm.

\textit{Contrast to our work:}
Since existing uncertainty guarantee-based methods typically assume stationary correlations and convex geometries, they often fail in spatially varying or obstacle-dense environments. We extend these guarantees to expressive, non-stationary models and non-convex domains, providing formal certificates for both near-optimal approximation and estimation uncertainty.

%% file: preliminaries.tex
\section{Preliminaries: Gaussian Processes}
\label{sec:prelim_gp}
We model the monitored quantity as an unknown function over a domain and use Gaussian processes~(GPs)~\cite{RasmussenW05} to represent both spatial structure and predictive uncertainty.

\textbf{GP Prior:}
A Gaussian process is a stochastic process such that any finite collection of function values is jointly Gaussian. Equivalently, a GP induces a distribution over functions, making it an expressive Bayesian prior for regression.
We place a GP prior on the latent data field $f(\cdot)$:
\begin{equation}
f(x) \sim \mathcal{GP}\big(m(x), k(x,x')\big),
\end{equation}
where $m(x)$ is the mean function (often taken to be $0$ after normalization) and kernel $k(x,x')$ is a positive semi-definite covariance function. The kernel specifies prior correlations:
\begin{equation}
k(x,x') = \mathrm{cov}\!\left(f(x), f(x')\right).
\end{equation}

\textbf{Noisy Observations and Posterior Prediction:}
Given $n$ noisy samples $\mathcal{D}=\{(x_i,y_i)\}_{i=1}^n$ with
$y_i = f(x_i) + \varepsilon_i$ and i.i.d.\ Gaussian noise $\varepsilon_i \sim \mathcal{N}(0,\sigma_n^2)$, define the stacked training inputs $X=[x_1,\dots,x_n]$ and outputs $y=[y_1,\dots,y_n]^\top$. For a query location $x_\ast$, GP regression yields a Gaussian predictive distribution
$p(f_\ast \mid X,y,x_\ast)=\mathcal{N}(\mu_\ast,\sigma_\ast^2)$, with
\begin{align}
\mu_\ast &= K_{\ast n}\big(K_{nn}+\sigma_n^2 I\big)^{-1} y, \\
\sigma_\ast^2 &= K_{\ast\ast} - K_{\ast n}\big(K_{nn}+\sigma_n^2 I\big)^{-1}K_{n\ast}.
\end{align}
Here, $K_{nn}$ is the $n\times n$ covariance matrix with $(K_{nn})_{ij}=k(x_i,x_j)$, $K_{\ast n}$ is the $1\times n$ vector with $(K_{\ast n})_i = k(x_\ast,x_i)$, $K_{n\ast}=K_{\ast n}^\top$, and $K_{\ast\ast}=k(x_\ast,x_\ast)$.
The posterior variance $\sigma_\ast^2$ provides an intrinsic uncertainty measure that depends only on the input locations (and the GP hyperparameters), which makes it especially useful for planning.

%% file: problem_formulation.tex
\section{Problem: Informative Path Planning with Guaranteed Estimation Uncertainty}

\textbf{System Setup:} 
A robot is tasked with monitoring an environmental field (e.g., salinity, temperature, bathymetry) over a purely spatial or spatio-temporal domain $\mathcal{X} \subset \mathbb{R}^d$. For planning and performance evaluation, we define a finite set of \emph{evaluation points} $\mathcal{V} = \{x_1, \dots, x_N\} \subset \mathcal{X}$ where accurate estimates are required.

Due to accessibility, safety, and operational constraints, the robot can only sample from a restricted set of \emph{candidate sensing locations} $\mathcal{C} = \{c_1, \dots, c_M\} \subseteq \mathcal{X}$. A feasible path is represented by an ordered sequence of waypoints $\mathcal{P} = (p_1, \dots, p_T)$, where each $p_t \in \mathcal{C}$. We denote the travel cost of a path as $\mathrm{cost}(\mathcal{P})$ (e.g., Euclidean distance or boundary-constrained path length).

\textbf{GP-based Uncertainty Surrogate:}
Directly optimizing the pointwise Mean Squared Error (MSE) is challenging because it generally depends on the unknown measurement values realized along $\mathcal{P}$. To obtain a tractable objective, we model the unknown field $f: \mathcal{X} \to \mathbb{R}$ as a Gaussian Process (GP).

The GP provides a posterior variance $\sigma^2_{\mathrm{post}}(x_i | \mathcal{P})$ that depends only on the sensing locations, not the observed values. Since the posterior variance provides a lower bound on the pointwise MSE~\cite{WagbergZSS17}, we define the \emph{worst-case uncertainty} as:
\begin{equation}
    J_{\max}(\mathcal{P}) := \max_{x_i \in \mathcal{V}} \sigma^2_{\mathrm{post}}(x_i | \mathcal{P}).
\end{equation}
The GP is initialized using a representative set of labeled points obtained from pilot missions, historical data, or coarse satellite observations.

\textbf{Problem Statement:}
The \emph{informative path planning problem with guaranteed estimation uncertainty} aims to find the minimum-cost path that ensures the posterior variance at every evaluation point remains below a user-specified threshold $\sigma^2_{\mathrm{tar}}$:
\begin{equation}
\begin{aligned}
    \mathcal{P}^\star \in \arg\min_{\mathcal{P} \in \mathcal{C}^*} \quad & \mathrm{cost}(\mathcal{P}) \\
    \text{s.t.} \quad & J_{\max}(\mathcal{P}) \le \sigma^2_{\mathrm{tar}}, \\
    & \mathrm{cost}(\mathcal{P}) \le D \quad \text{(optional),}
\end{aligned}
\label{eq:ipp_guarantee}
\end{equation}
where $\mathcal{C}^* := \bigcup_{T \ge 1} \mathcal{C}^T$ denotes the set of all finite waypoint sequences over $\mathcal{C}$, and $D$ is an optional budget constraining the total path cost.

%% file: method.tex
\section{Near-Optimal Informative Path Planning}
\label{sec:method}

This section presents our approach to solving the informative path planning problem with guaranteed estimation uncertainty. We first describe how we leverage available prior information to compute coverage maps. We then use these maps in two algorithms—\textsc{GreedyCover} and \textsc{GCBCover}—which address sensing-location selection and the joint problem of sensing-location selection and route planning, respectively.

\subsection{Coverage Map Construction Theory}

We begin by fitting a Gaussian process~(GP) to the available prior information. Depending on the data source, we use either a non-stationary kernel~\cite{ChenKL22} or a physics-informed kernel~\cite{HarkonenLR}.
Using the learned GP, we construct \emph{coverage maps} that specify which evaluation points
$v \in \mathcal{V}$ can be driven below the user-specified target posterior variance
$\sigma^2_{\mathrm{tar}}$ by taking a measurement at a candidate sensing location
$c \in \mathcal{C}$.

Ideally, coverage would be evaluated jointly over \emph{sets} of sensing locations, since the posterior variance at an evaluation
point depends on all measurements collected from a path. However, informative path planning is NP-hard~\cite{BinneyKS13};
exact evaluation would require computing posterior variances for many subsets of the power set of $\mathcal{C}$, which is
computationally infeasible. We therefore adopt a scalable alternative: for each candidate location $c \in \mathcal{C}$, we compute the set of
evaluation points that would satisfy the variance threshold if we were to observe \emph{only} at $c$. These per-location coverage maps
can then be combined efficiently during planning.

We justify this construction with the following results. The first provides an explicit prior covariance threshold under which a single
measurement at a sensing location can reduce the posterior variance at a given evaluation point below $\sigma^2_{\mathrm{tar}}$. The
second formalizes the monotonicity property that conditioning on additional measurements can only further reduce posterior variance;
hence, the single-location condition yields a conservative coverage region for multi-measurement paths.

\begin{theorem}[Minimum Required Prior Covariance]
\label{thm:single_point} $\\$
Let $f \sim \mathcal{GP}(0,k)$ with independent Gaussian observation noise variance $\sigma_n^2$, and consider a single noisy observation at candidate location $c \in \mathcal{C}$. 
Fix an evaluation location $v \in \mathcal{V}$. 
For any target posterior variance $\sigma^2_{\mathrm{tar}} \in (0,k(v,v))$, the condition
\[
\sigma_{\mathrm{post}}^2(v \mid c) \le \sigma^2_{\mathrm{tar}}
\]
holds \emph{if and only if}
\begin{equation}
\label{eq:cov_threshold_detailed}
|k(c,v)|
\;\ge\;
\sqrt{\bigl(k(v,v) - \sigma^2_{\mathrm{tar}}\bigr)\bigl(k(c,c) + \sigma_n^2\bigr)}.
\end{equation}
Thus, achieving a posterior variance of at most $\sigma^2_{\mathrm{tar}}$ at $v$ requires that the prior covariance between $c$ and $v$
exceed the threshold in~Equation~\ref{eq:cov_threshold_detailed}.
\end{theorem}

\noindent Please refer to the Appendix for the proof.

\begin{theorem}[Variance Reduction with Multiple Locations]
\label{thm:multi_point} $\\$
Under the setup of Theorem~\ref{thm:single_point}, consider $n$ sensing locations $c_1,\dots,c_n \in \mathcal{C}$.
Fix an evaluation location $v \in \mathcal{V}$ and a target variance $\sigma^2_{\mathrm{tar}} \in (0,k(v,v))$. Suppose there exists
$j \in \{1,\dots,n\}$ such that
\[
|k(c_j,v)|
\;\ge\;
\sqrt{\bigl(k(v,v) - \sigma^2_{\mathrm{tar}}\bigr)\bigl(k(c_j,c_j) + \sigma_n^2\bigr)}.
\]
Then the posterior variance at $v$ conditioned on \emph{all} observations satisfies
\[
\sigma_{\mathrm{post}}^2\!\left(v \mid c_{1:n}\right) \le \sigma^2_{\mathrm{tar}}.
\]
\end{theorem}

\noindent Please refer to the Appendix for the proof.

\begin{figure}[t]
    \centering
    \includegraphics[width=\linewidth]{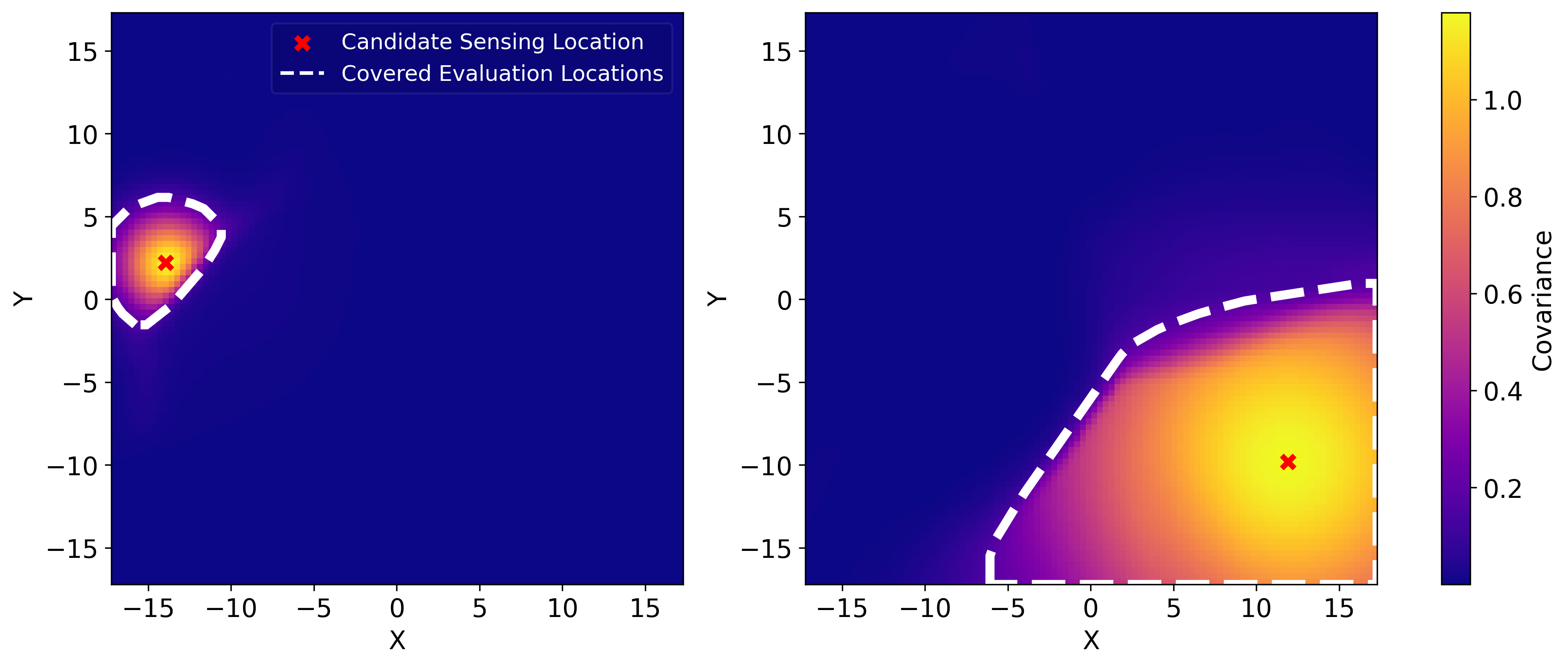}
    \caption{Coverage maps for two candidate sensing locations under a non-stationary kernel. Each dotted white blob indicates the set of evaluation locations for which that candidate location can guarantee the target posterior-variance threshold.}
    \label{fig:fovs}
\end{figure}

\subsection{Binary Coverage Map Construction}
\label{sec:coverage_map_construction}

Together, Theorems~\ref{thm:single_point} and~\ref{thm:multi_point} motivate a per-location notion of coverage: a sensing location $c_j$ covers an evaluation point $v_i$ if the prior covariance $k(c_j,v_i)$ exceeds the threshold in Equation~\ref{eq:cov_threshold_detailed}. We encode this relationship in a binary matrix $B \in \{0,1\}^{M\times N}$. For each candidate sensing location $c_j \in \mathcal{C}$ and evaluation point $v_i \in \mathcal{V}$, we set $B_{ji} = 1$ (i.e., sensing at $c_j$ covers $v_i$) if

\[
\mathbb{I}\!\left\{
|k(c_j,v_i)| \;\ge\;
\sqrt{\left(k(v_i,v_i)-\sigma^2_{\mathrm{tar}}\right)\;
\left(k(c_j,c_j)+\sigma_n^2\right)}
\right\},
\]

\noindent and $B_{ji} = 0$ otherwise . Figure~\ref{fig:fovs} shows example coverages.

This binary representation provides two practical benefits. First, it reduces memory compared to storing real-valued covariances for
all pairs in $\mathcal{C}\times\mathcal{V}$. Second, it enables fast planning-time updates via set operations: the evaluation points
covered by a path can be computed by unions over the rows of $B$ corresponding to the visited candidate locations, without repeatedly
evaluating GP posterior variances for different candidate subsets.

\subsection{\textsc{GreedyCover}: IPP via Greedy Selection and TSP}
\label{subsec:path-greedy-tsp}

Given the binary coverage matrix $B$ from Section~\ref{sec:coverage_map_construction}, our first planning approach proceeds in two
stages: (i) greedily select a subset of candidate sensing locations that maximizes the number of covered evaluation points, and
(ii) construct a path that visits all selected locations using a TSP solver~\cite{Christofides22}.

\subsubsection{Greedy Sensor Selection}

Starting from an empty selected set $S_0=\emptyset$ and an empty covered set $U_0=\emptyset$, the greedy procedure iteratively adds
the candidate location that covers the largest number of \emph{new} (previously uncovered) evaluation points. Let $S_t \subseteq
\{1,\dots,M\}$ denote the indices of selected candidates after iteration $t$, and let $U_t \subseteq \mathcal{V}$ denote the set of
evaluation points covered by $S_t$. The marginal gain of adding candidate $j \notin S_t$ is
\[
\Delta(j \mid S_t)
:= \bigl|\{v_i \in \mathcal{V} : B_{ji}=1,\; v_i \notin U_t\}\bigr|.
\]
At each iteration, we select
\[
j^\star \in \arg\max_{j \notin S_t} \Delta(j \mid S_t),
\]
update $S_{t+1} = S_t \cup \{j^\star\}$, and expand the covered set
\[
U_{t+1} = U_t \cup \{v_i \in \mathcal{V} : B_{j^\star i}=1\}.
\]
The loop terminates once all evaluation points are covered (i.e., $|U_t|=N$) or once no remaining candidate yields additional
coverage (i.e., $\max_{j \notin S_t}\Delta(j\mid S_t)=0$). Algorithm~\ref{alg:greedy-coverage} in the Appendix summarizes this procedure. The worst-case running time of the algorithm is $\mathcal{O}(M^2N)$, when $M \le N$.

\subsubsection{Submodularity and Approximation Guarantees}

The objective implicitly optimized by the above greedy method is a classic monotone submodular set function~\cite{NemhauserWF78,Wolsey82, krauseSG08}.
Define the \emph{coverage function} with $S \subseteq \{1,\dots,M\}$:

\begin{equation}
\label{eq:coverage_objective}
F(S)
:=
\left|\bigcup_{j \in S}\{v_i \in \mathcal{V} : B_{ji}=1\}\right|,
\end{equation}

\noindent which counts how many evaluation points are covered by the selected candidates. This function is:

\begin{itemize}
\item \emph{Monotonically Increasing}: If $Y \subseteq Z$, then $F(Y) \le F(Z)$, since adding sensing locations can only enlarge the union of covered points.
\item \emph{Submodular}: It satisfies diminishing returns. For all $Y \subseteq Z$ and any location $l \notin Z$,
\[
F(Y \cup \{l\}) - F(Y)
\;\ge\;
F(Z \cup \{l\}) - F(Z),
\]
since the set of newly covered points contributed by $l$ can only shrink as the baseline set grows.
\end{itemize}

The \textsc{GreedyCover} procedure addresses the \emph{submodular set cover} problem: minimizing the number of sensing locations required to reach coverage $F(\mathcal{V})$. Given that $F(S)$ is an integer-valued monotone submodular function, the result by Wolsey~\cite{Wolsey82} ensures that the greedy solution $S$ satisfies
\[
F(S)
\;\le\;
\textrm{OPT} \cdot \left( 1 + \ln \max_{j \in \{1,\dots,M\}} F(\{j\}) \right),
\]
where $\textrm{OPT}$ denotes the minimum number of locations required to cover all evaluation points $v_i \in \mathcal{V}$. Thus, the greedy method provides a provably near-optimal selection for minimizing the sensing resources necessary to satisfy the coverage objective.

\subsubsection{Path Construction via TSP}

Given the selected index set $S$, we construct a low-cost route that visits all sensing locations $\{c_j : j \in S\}$. We treat these
locations as nodes in a complete Euclidean graph, with edge weights given by pairwise distances between locations. We then invoke a
TSP solver (e.g., the Christofides algorithm~\cite{Christofides22}) to compute a path through the nodes, which we use as the execution path.

While our formulation constrains sensing locations to the valid environment—including non-convex domains—the straight-line edges of a Euclidean TSP path may still traverse infeasible regions, such as obstacles or areas outside the boundary. To ensure feasibility, we post-process the path using standard motion-planning techniques (e.g., grid-based or sampling-based planners~\cite{LynchP17}) to replace straight-line segments with collision-free, admissible paths.

\subsubsection{Limitations of the Decoupled Approach}

This two-stage pipeline—first selecting sensing locations to satisfy the coverage requirement, then routing them with a TSP solver—has an important limitation: it provides \emph{no joint approximation guarantee} for the combined information--travel objective. The submodular approximation guarantees provided by Wolsey~\cite{Wolsey82} apply only to the selection of sensing locations and do not account for path cost. In practice, greedy selection may choose sensing locations that are geographically dispersed, leading to unnecessarily long routes even though other, slightly different selections could satisfy the estimation uncertainty requirements at a lower travel cost. This limitation motivates our next method, which more tightly couples sensing-location selection and routing.


\subsection{\textsc{GCBCover}: IPP via Generalized Cost--Benefit Algorithm}
\label{subsec:path-gcb}

Our second planning method couples sensing-location selection and routing via the generalized cost--benefit~(GCB)
algorithm~\cite{ZhangV16,LinT23}. Unlike the decoupled approach, which first maximizes coverage and then computes a path,
GCB explicitly trades off \emph{marginal coverage gain} against the \emph{marginal increase in path cost} at each iteration. 

\subsubsection{Generalized Cost--Benefit Selection with Coverage Maps}

We reuse the binary coverage matrix \(B \in \{0,1\}^{M \times N}\) from
Section~\ref{sec:coverage_map_construction} and the coverage function in
Equation~\ref{eq:coverage_objective}. Let \(\mathrm{cost}(S)\) denote the routing cost of a selected index set
\(S \subseteq \{1,\dots,M\}\), defined as the length of a (near-)shortest route that visits the sensing locations
\(\{c_j : j \in S\}\).

At iteration $t$, given the current sensing set $S_t$ with coverage $F(S_t)$ and routing cost $\mathrm{cost}(S_t)$, each remaining candidate
$j \notin S_t$ is scored by a cost--benefit ratio constructed from:\\
(i) the \emph{marginal coverage gain}:
\[
\Delta F(j \mid S_t) := F(S_t \cup \{j\}) - F(S_t),
\]
and (ii) the \emph{marginal routing-cost increment}:
\[
\Delta \mathrm{cost}(j \mid S_t) := \mathrm{cost}(S_t \cup \{j\}) - \mathrm{cost}(S_t).
\]
The candidate with the
largest ratio:

\[
\Delta F(j \mid S_t)/\Delta \mathrm{cost}(j \mid S_t)
\]

\noindent is selected, provided it satisfies the routing budget $D$, effectively prioritizing candidates that yield substantial new coverage for a minimal increase in routing cost.

Because a purely ratio-greedy strategy can be suboptimal under a strict budget, GCB also computes a standard greedy coverage solution and truncates it to satisfy the budget. The algorithm then returns whichever of the budget-feasible solutions attains higher coverage. This comparison step is also necessary for the approximation guarantee stated below. Algorithm~\ref{alg:gcb} in the Appendix summarizes the GCB algorithm.

The worst-case running time is dominated by repeated routing cost evaluations. Let \(T(n)\) denote the time required to compute the routing cost for \(n\) sensing locations (e.g., via a TSP solver). Using the Christofides algorithm~\cite{Christofides22}, where $T(n) = \mathcal{O}(n^3)$, the resulting worst-case complexity of GCB is $\mathcal{O}(M^2(N + M^3))$ for $M \le N$.

In our experiments, we use a faster heuristic TSP routine. Give the current solution path, we estimate the routing cost increase upon adding a candidate sensing location using a fast heuristic (see Algorithm~\ref{alg:approx-route-increment} in the Appendix). Once a location is selected via the GCB ratio, we invoke a full TSP solver to refine the path. This hybrid strategy significantly reduces runtime while maintaining high solution quality.

\subsubsection{Approximation Guarantees Under Routing Constraints}

As established in Section~\ref{subsec:path-greedy-tsp}, the coverage objective $F$ is monotone submodular. Similarly, the routing cost $\mathrm{cost}(\cdot)$ is a monotone set function defined by the length of a TSP path through the selected sensing locations~\cite{ZhangV16}. Consequently, the problem can be framed as a submodular maximization task subject to a routing constraint:
\[
\max_{S \subseteq \{1,\dots,M\}} F(S)
\quad \text{s.t.} \quad
\mathrm{cost}(S) \le D.
\]

The generalized cost--benefit~(GCB) algorithm provides a constant-factor approximation guarantee for this class of problems. In particular, under
mild conditions on the routing cost function---namely, approximate submodularity with submodularity ratio $\alpha_c$, bounded
curvature $\kappa_c$, and a $\psi(n)$-approximation factor for the TSP solver---the GCB solution satisfies the following
bi-criterion guarantee~\cite{ZhangV16}:
\[
F(S)
\;\ge\;
\frac{1}{2}\Bigl(1 - \frac{1}{e}\Bigr)\, \mathrm{OPT},
\]
where $\mathrm{OPT}$ denotes the optimal coverage achievable under a slightly tighter effective budget (with the tightening
depending on $\alpha_c$, $\kappa_c$, $\psi(n)$, and the largest feasible solution size). Equivalently, GCB achieves at least a
$\tfrac{1}{2}(1-1/e)\approx 31.6\%$ fraction of an \emph{approximate} optimum, where the approximation reflects the quality of the TSP solver. We refer to Zhang and Vorobeychik~\cite{ZhangV16} for a more detailed explanation of the approximation guarantee.

In our setting, $F$ is exactly monotone submodular; thus, all approximation effects arise from the routing component (i.e., the
submodularity ratio/curvature of $\mathrm{cost}(\cdot)$ and the approximation quality of the TSP solver). Consequently, our method inherits
the same worst-case approximation factor as the underlying GCB framework while directly optimizing a coverage objective tailored to
informative sensing.

%% file: experiments.tex
\section{Experiments}

\begin{figure}[ht!]
    \centering
    \includegraphics[width=\linewidth]{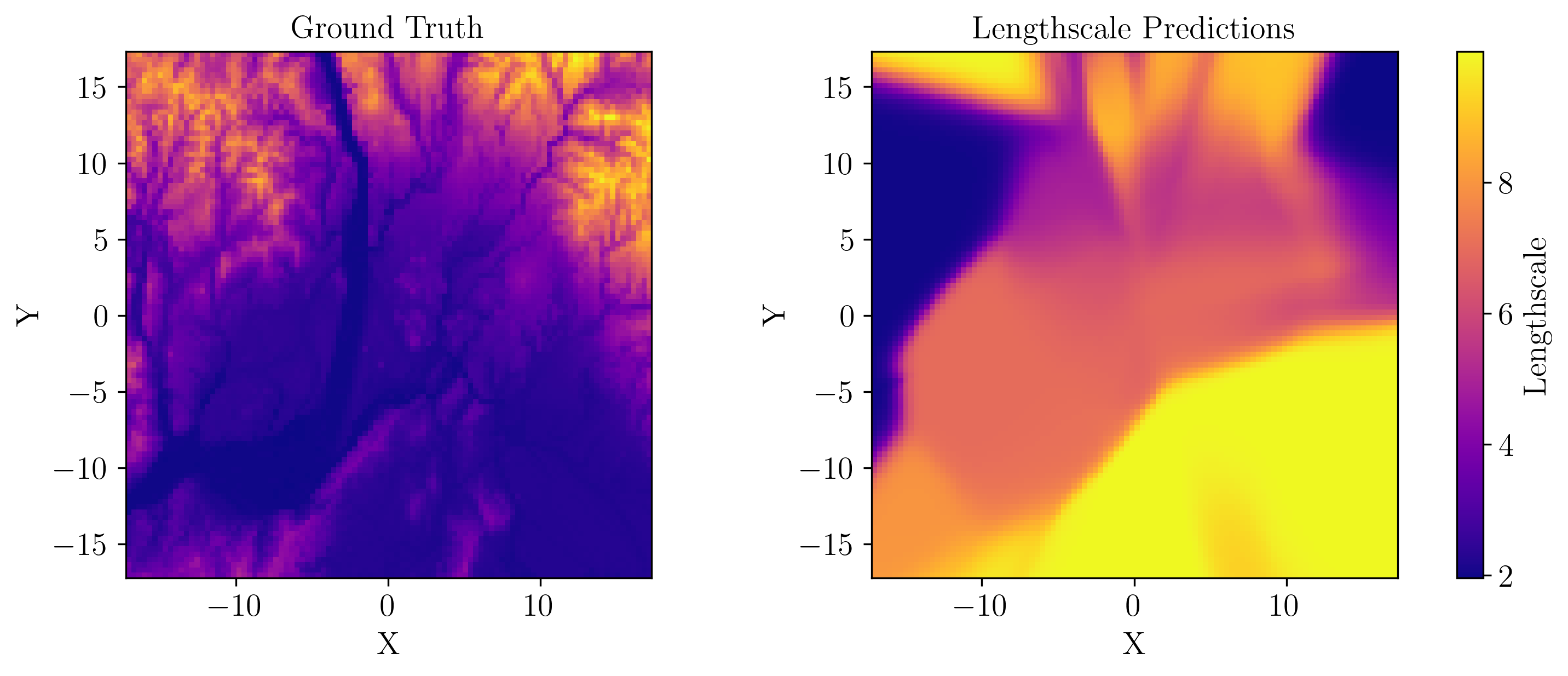}
    \caption{Left: SRTM ground truth data $(45^\circ\mathrm{N},\,123^\circ\mathrm{W})$. \\Right: Lengthscales from a non-stationary kernel.}
    \label{fig:gt}
\end{figure}

\begin{figure*}[!ht]
   \centering
    \begin{subfigure}{0.32\textwidth}
        \includegraphics[width=\textwidth]{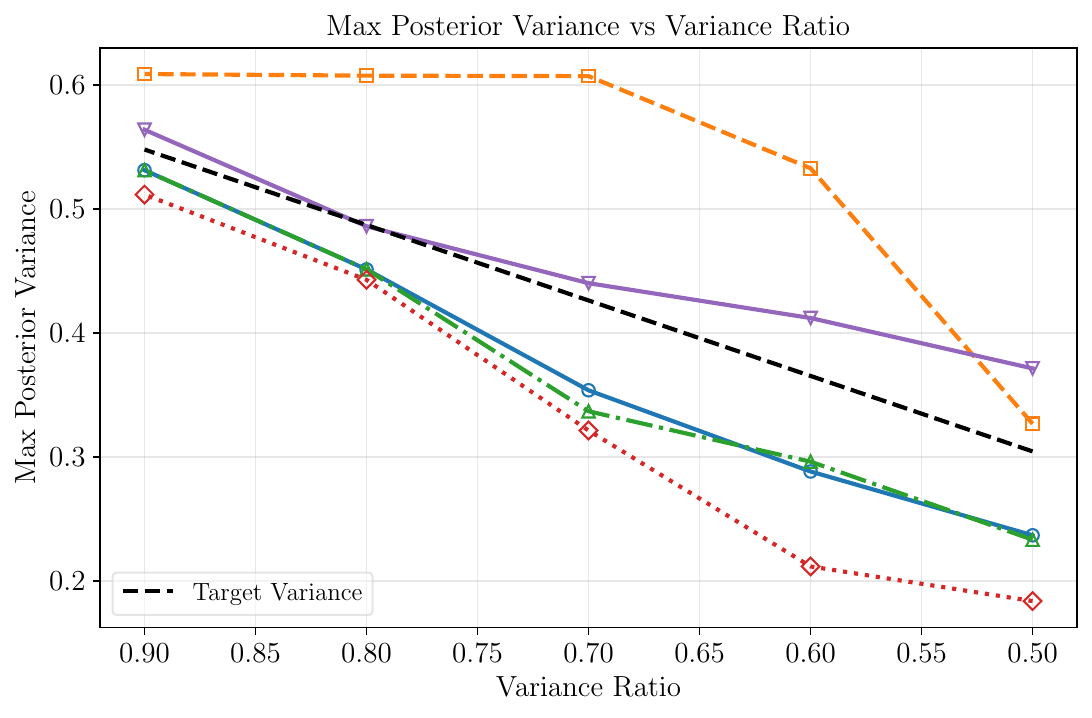}
    \end{subfigure}
    \hfill
    \begin{subfigure}{0.32\textwidth}
        \includegraphics[width=\textwidth]{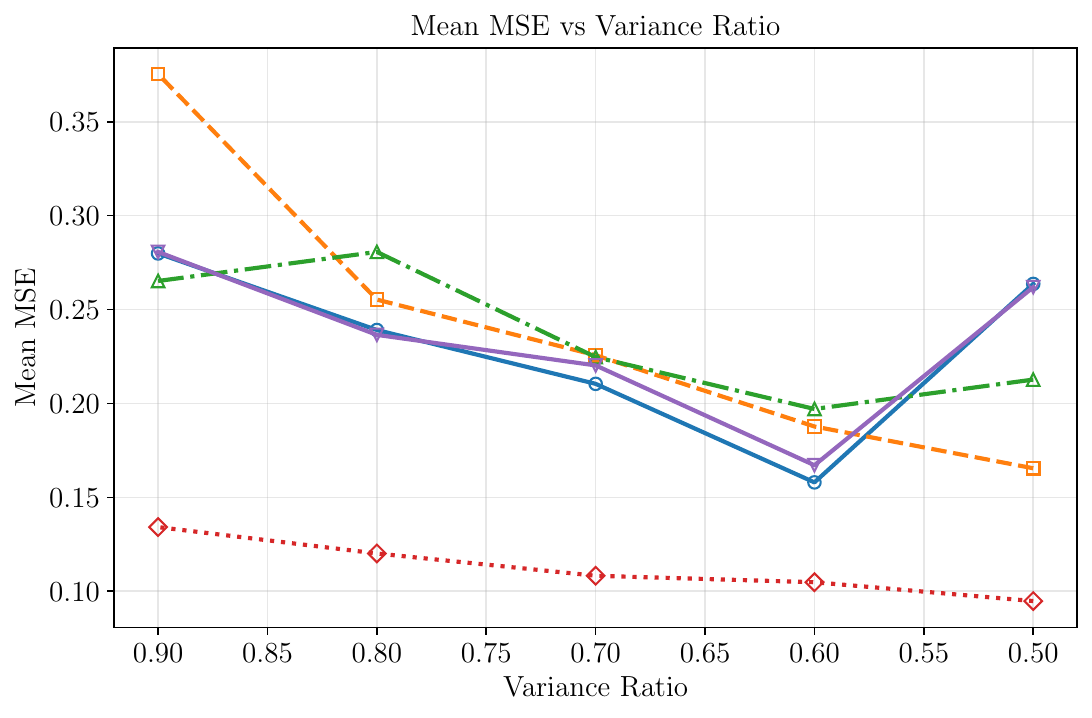}
    \end{subfigure}
    \hfill
    \begin{subfigure}{0.32\textwidth}
        \includegraphics[width=\textwidth]{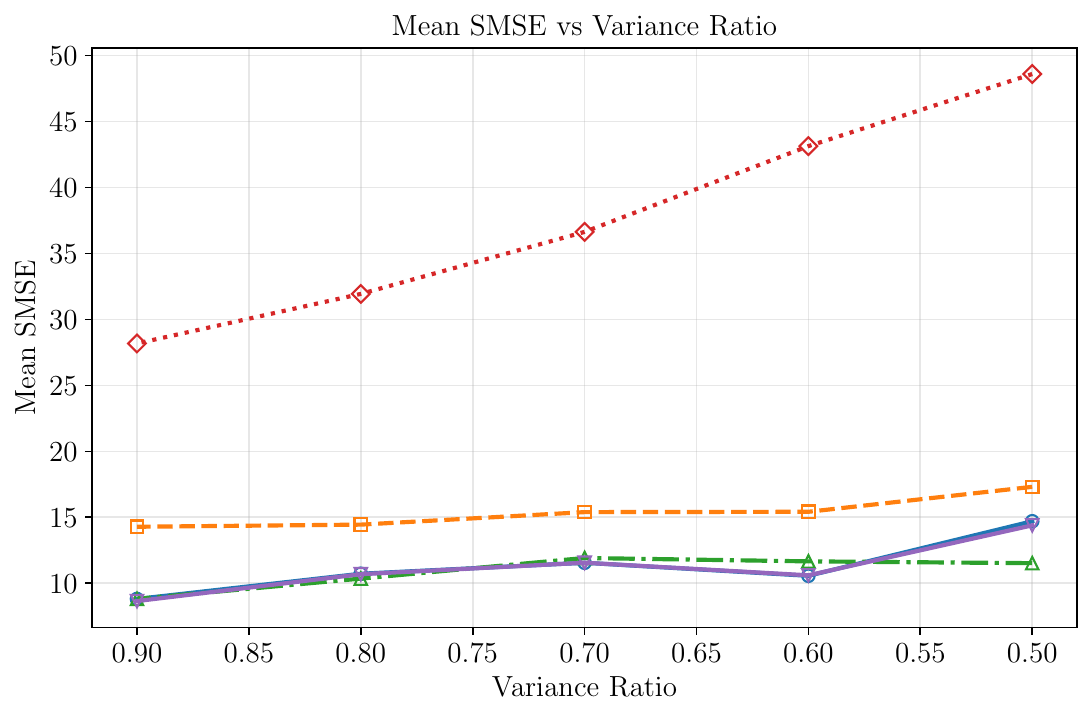}
    \end{subfigure}
    \hfill
    \begin{subfigure}{0.32\textwidth}
        \includegraphics[width=\textwidth]{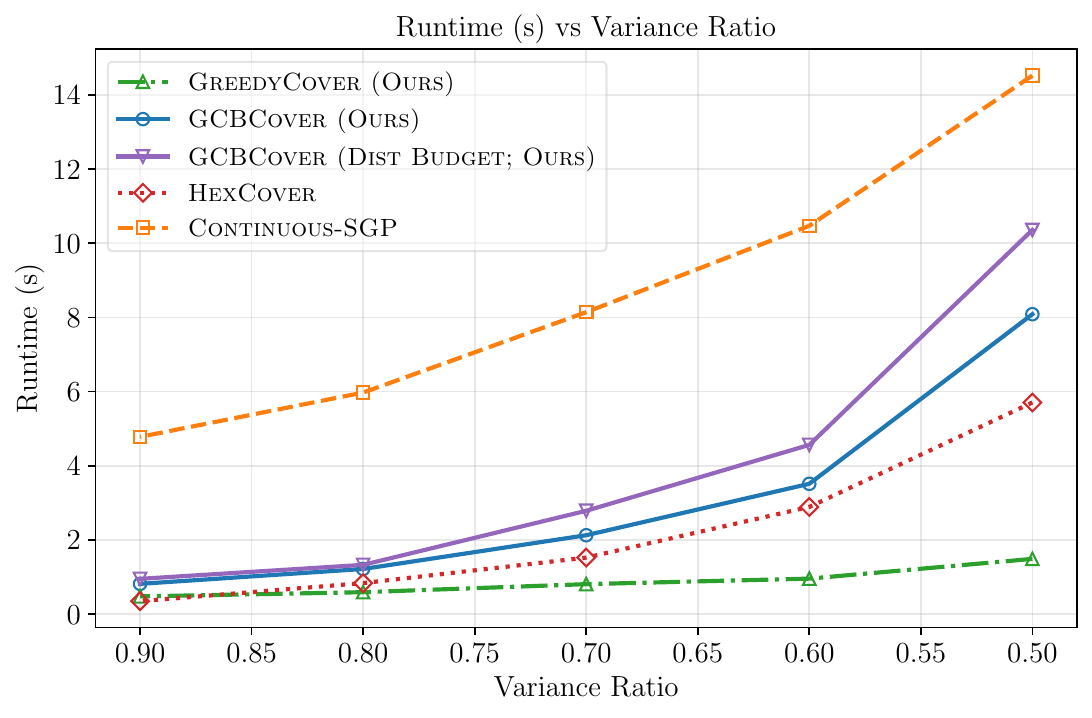}
    \end{subfigure}
    \hfill
    \begin{subfigure}{0.32\textwidth}
        \includegraphics[width=\textwidth]{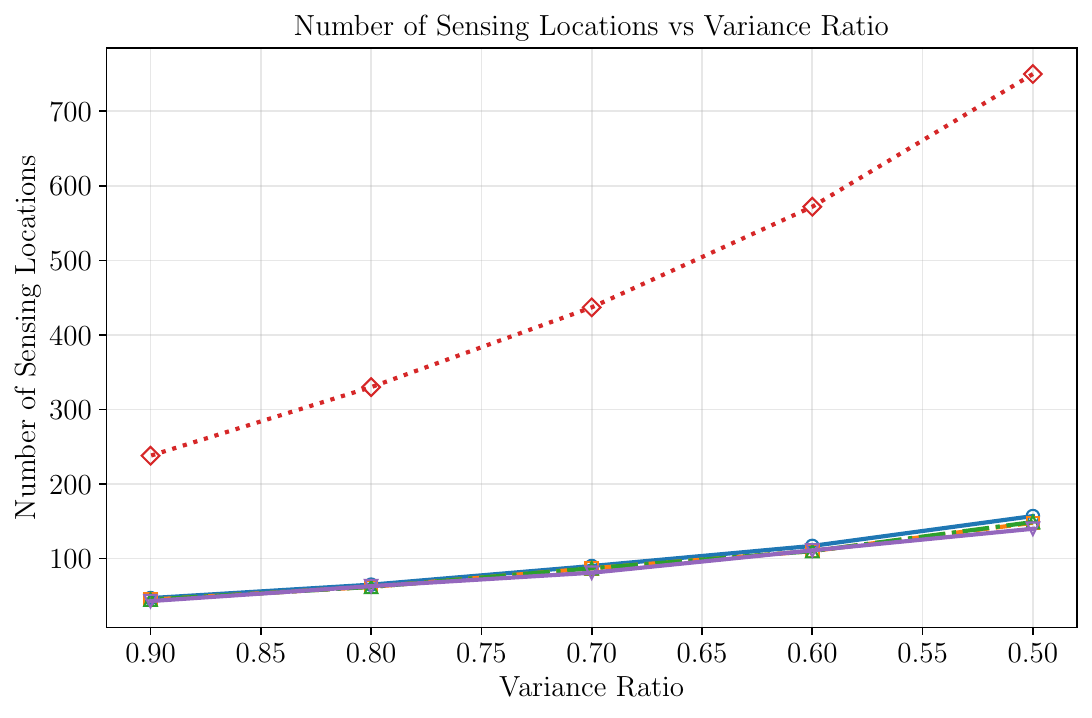}
    \end{subfigure}
    \hfill
    \begin{subfigure}{0.32\textwidth}
        \includegraphics[width=\textwidth]{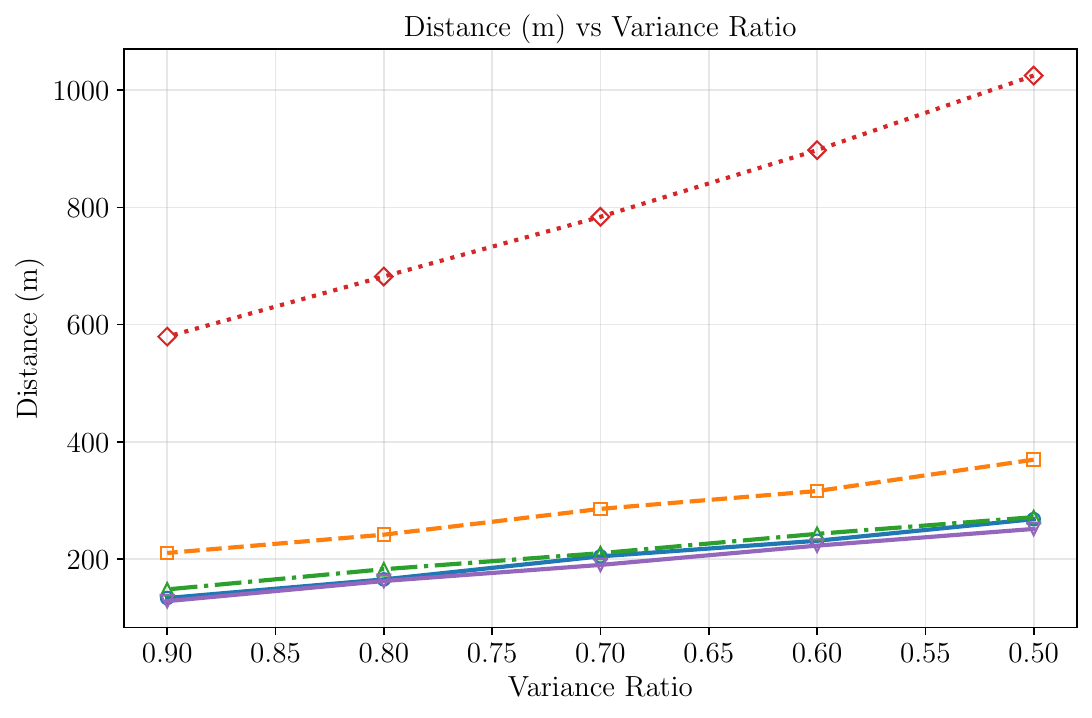}
    \end{subfigure}
    \caption{Benchmark results on the SRTM dataset $(45^\circ\mathrm{N},\,123^\circ\mathrm{W})$. Top row: Max posterior variance, MSE, and SMSE. Bottom row: Runtime, sensing locations, and path length. \textsc{GreedyCover} and \textsc{GCBCover} satisfy uncertainty targets with fewer sensing locations and shorter paths than \textsc{HexCover} and \textsc{Continuous-SGP}. In budgeted scenarios, \textsc{GCBCover} prioritizes strict adherence to travel limits over further variance reduction. Lower values represent better performance across all metrics.} 
    \label{fig:benchmark}
\end{figure*}

This section evaluates the proposed method through two experimental studies. First, we benchmark our methods against representative baseline approaches using real-world topographic datasets. Second, we validate the practical feasibility of our approach through field deployments.

\subsection{Benchmark Experiments}

\begin{figure*}[!ht]
   \centering
    \begin{subfigure}{0.89\textwidth}
        \includegraphics[width=\textwidth]{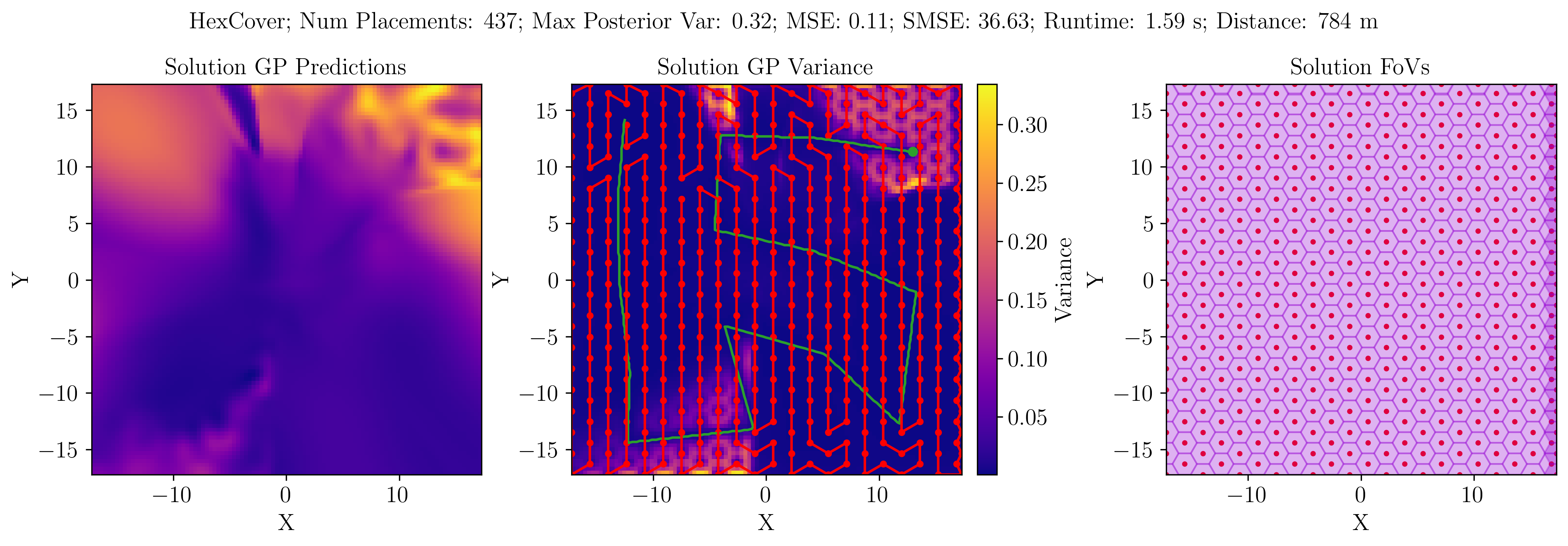}
    \end{subfigure}
    \hfill
    \begin{subfigure}{0.89\textwidth}
        \includegraphics[width=\textwidth]{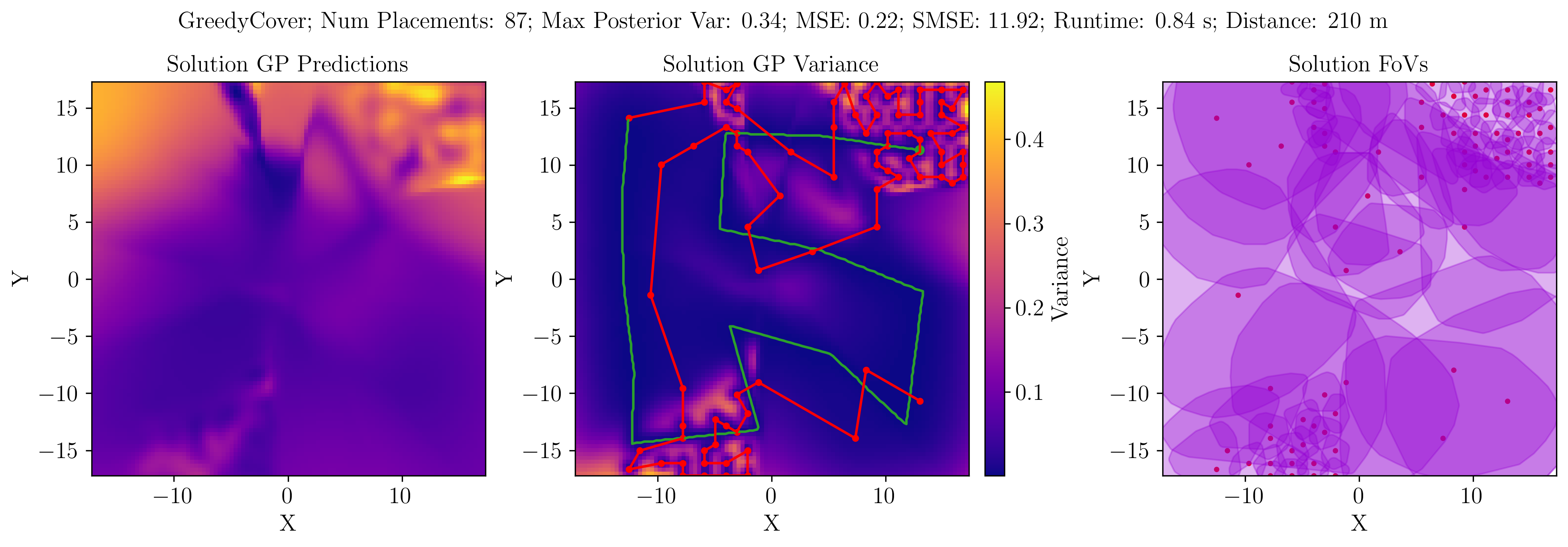}
    \end{subfigure}
    \hfill
    \begin{subfigure}{0.89\textwidth}
        \includegraphics[width=\textwidth]{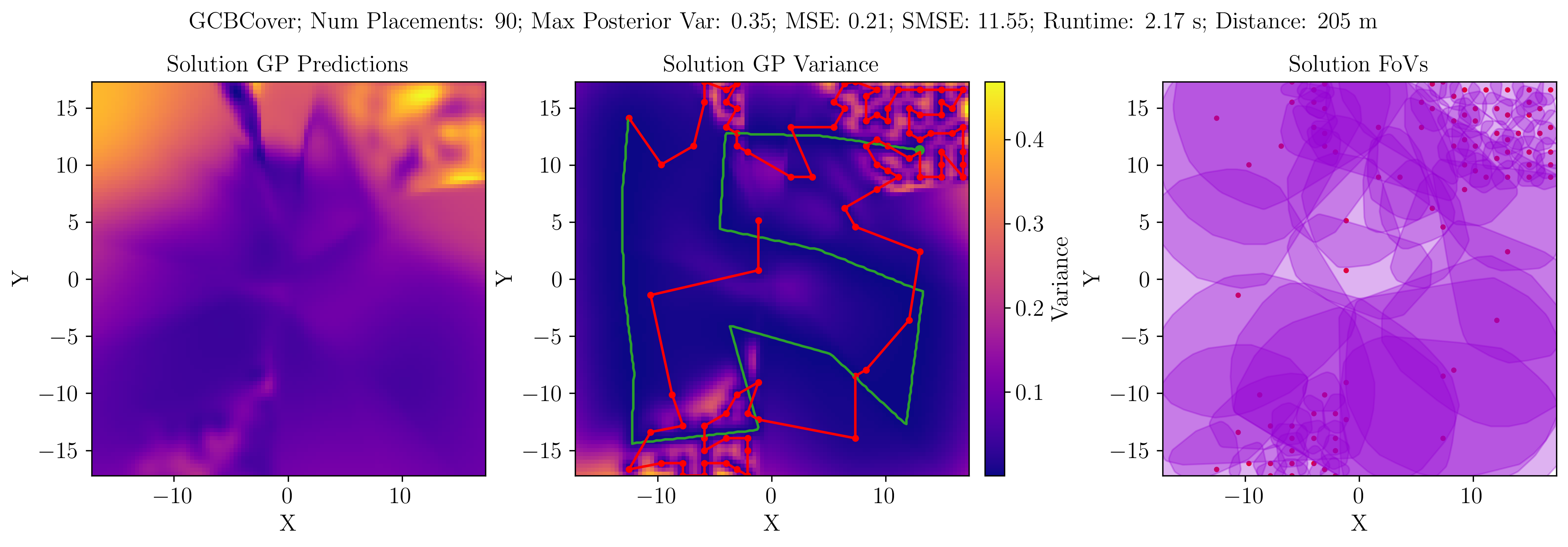}
    \end{subfigure}
    \hfill
    \begin{subfigure}{0.89\textwidth}
        \includegraphics[width=\textwidth]{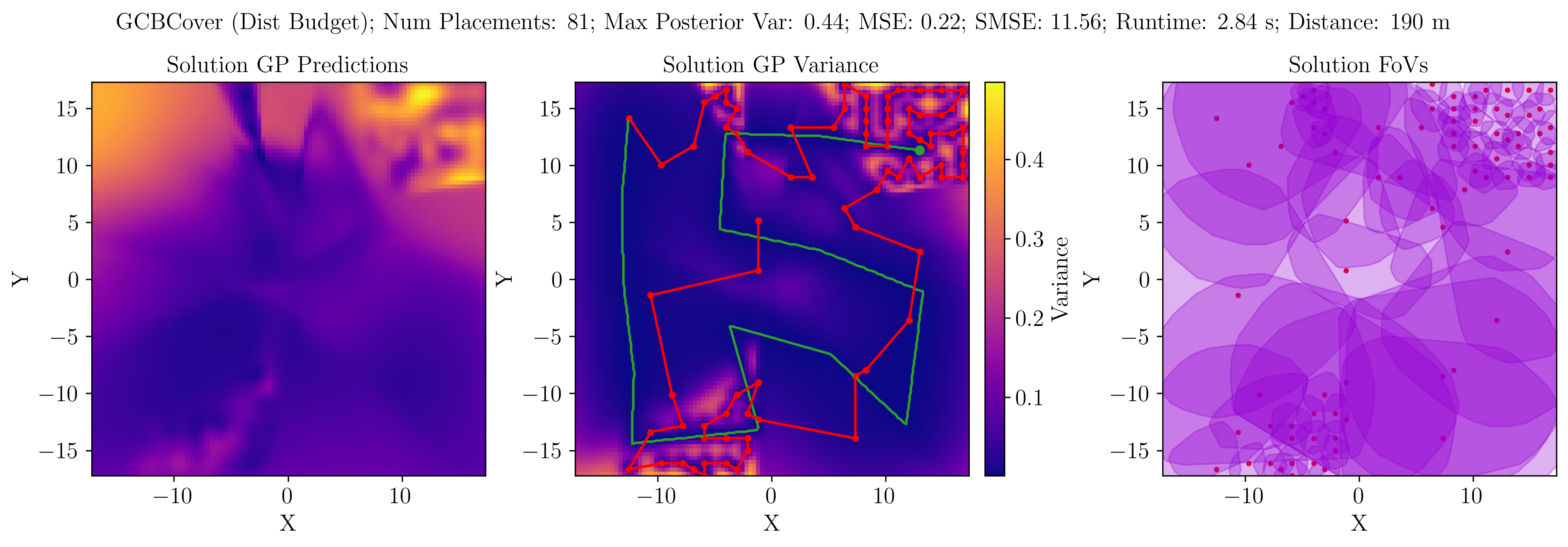}
    \end{subfigure}
    \caption{Comparison of solutions on the SRTM dataset $(45^\circ\mathrm{N},\,123^\circ\mathrm{W})$ for a 0.7 variance ratio (target variance threshold: 0.43). Red: informative paths and sensing locations; Green: initial pilot path; Violet: approximate sensor fields-of-view.}
    \label{fig:results}
\end{figure*}

We benchmarked the proposed methods on the Shuttle Radar Topography Mission~(SRTM) dataset~\cite{SRTM}, which provides high-resolution digital elevation data worldwide. We focused on the region near $(45^\circ\mathrm{N},\,123^\circ\mathrm{W})$, chosen for its complex, non-stationary spatial structure; results for three additional locations are reported in the Appendix. All benchmarks were run on a computer with an Intel i9-14900K CPU and 64~GB of RAM.

To model the environment, we employ a non-stationary Attentive kernel~\cite{ChenKL22}, which utilizes a mixture of Radial Basis Function (RBF) kernels to capture spatially varying correlations. We optimize the kernel hyperparameters by fitting a Gaussian process~(GP)~\cite{RasmussenW05} to data collected along an initial pilot path. To generate this path, we uniformly sample locations in the domain $\mathcal{X}$, cluster them using $k$-means~\cite{Bishop06}, and select 10 resulting centroids as waypoints. A Traveling Salesperson Problem~(TSP) solver~\cite{Christofides22} is then used to compute a route through them. After collecting 350 measurements along this pilot trajectory to initialize the GP, the learned kernel function and likelihood noise variance are provided to each planner to generate solution paths. Figure~\ref{fig:gt} shows the SRTM ground truth for $(45^\circ\mathrm{N},\,123^\circ\mathrm{W})$, along with the spatially varying lengthscales learned by the Attentive kernel.

We evaluated our proposed methods, \textsc{GreedyCover} (Section~\ref{subsec:path-greedy-tsp}) and \textsc{GCBCover} (Section~\ref{subsec:path-gcb}), considering \textsc{GCBCover} in both unconstrained and budgeted-distance settings. We compared these methods against two established baselines. The first, \textsc{HexCover}~\cite{DuttaWTS23}, generates IPP solutions with estimation uncertainty guarantees. Since \textsc{HexCover} requires a stationary RBF kernel, we set its lengthscale to the minimum value inferred by the trained Attentive kernel---a conservative choice that usually preserves the uncertainty guarantee. The second, \textsc{Continuous-SGP}~\cite{JakkalaA25b}, is a high-performance IPP method known to outperform mutual-information and variance-based approaches, though it lacks formal uncertainty guarantees.

None of the planning methods explicitly account for measurements collected from the pilot path; during planning, they use only the learned kernel and noise variance. \textsc{GreedyCover}, \textsc{GCBCover}, and \textsc{Continuous-SGP} can be warm-started by seeding the selected sensing locations with already visited points and then planning the remaining route, potentially reducing the number of additional sensing locations (and thus total travel distance). However, \textsc{HexCover} does not support such initialization. For a consistent comparison, we therefore incorporate the pilot measurements only during evaluation, not during planning.

We utilized a GP trained on pilot data to determine the maximum posterior variance across a discrete evaluation grid. We define the target variance threshold as a fixed fraction of this maximum, varying the fraction from $0.9$ to $0.5$. For the distance-budgeted configuration of \textsc{GCBCover}, the budget is defined as the length of its unconstrained solution minus $20\,\mathrm{m}$. This design choice anchors the budget to the physical requirements of the uncertainty constraint while imposing a deliberate shortfall to challenge the planner’s efficiency. Additionally, as \textsc{Continuous-SGP} requires the number of sensing locations to be specified a priori, we set this value to match the number of locations generated by the \textsc{GreedyCover} solution.

To evaluate each method, we combined measurements from the pilot path with those collected at the solution waypoints and computed six metrics: (i) maximum posterior variance, (ii) mean squared error (MSE)~\cite{RasmussenW05}, (iii) standardized mean squared error (SMSE)~\cite{RasmussenW05}, (iv) algorithm runtime, (v) number of selected sensing locations, and (vi) total solution path length. Figure~\ref{fig:benchmark} presents the benchmark results for $(45^\circ\mathrm{N},\,123^\circ\mathrm{W})$, while Figure~\ref{fig:results} illustrates representative solution paths. Additional results are available in the Appendix.

Across the full range of threshold ratios, all methods successfully satisfy the target variance requirements, with two notable exceptions: distance-budgeted \textsc{GCBCover}, which prioritizes the travel distance limit over variance reduction, and \textsc{Continuous-SGP}, which offers no formal guarantees.

\begin{figure*}[ht]
    \centering
    \begin{subfigure}{0.311\textwidth}
        \includegraphics[width=\textwidth]{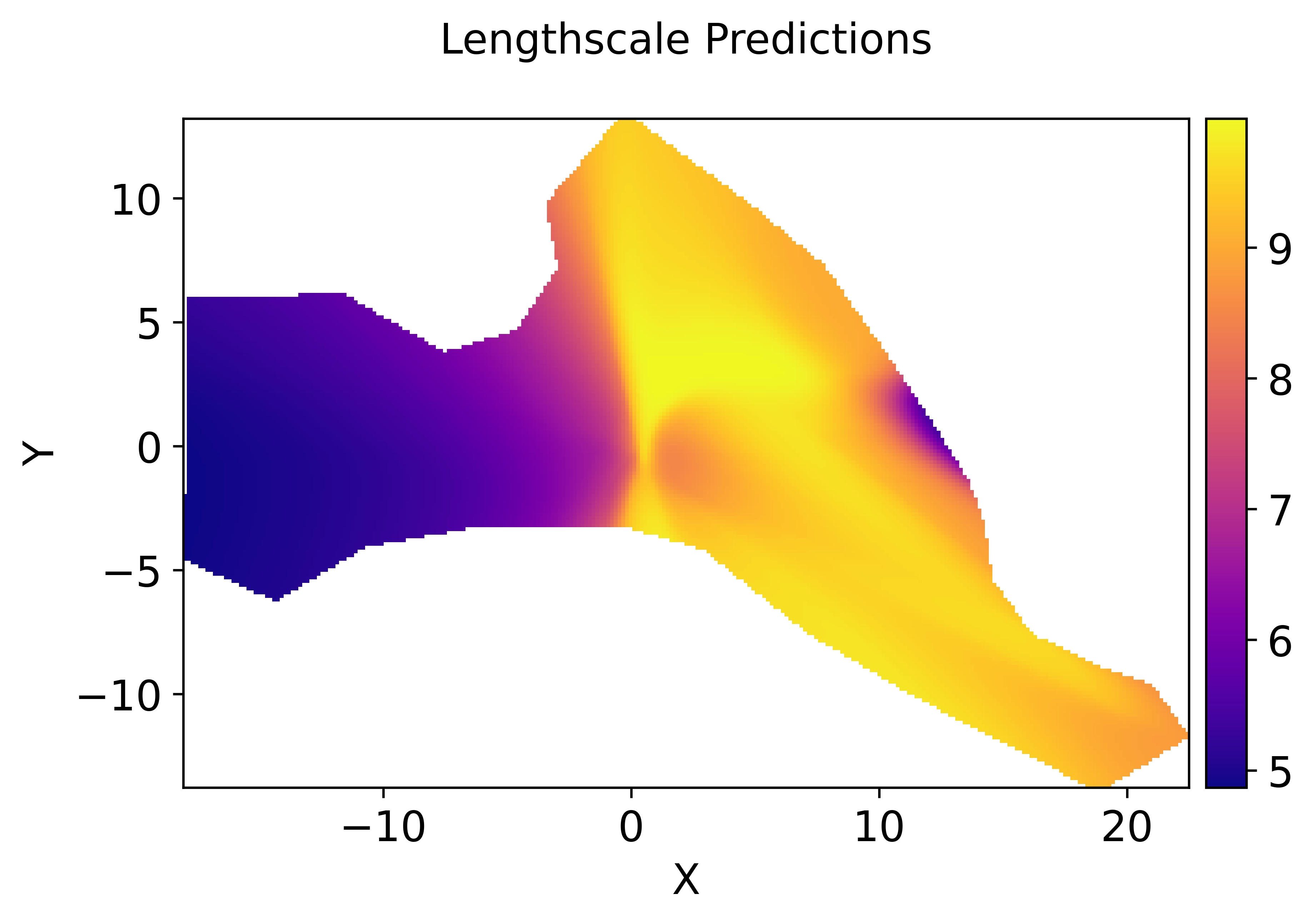}
    \end{subfigure}
    \hfill
    \begin{subfigure}{0.32\textwidth}
        \includegraphics[width=\textwidth]{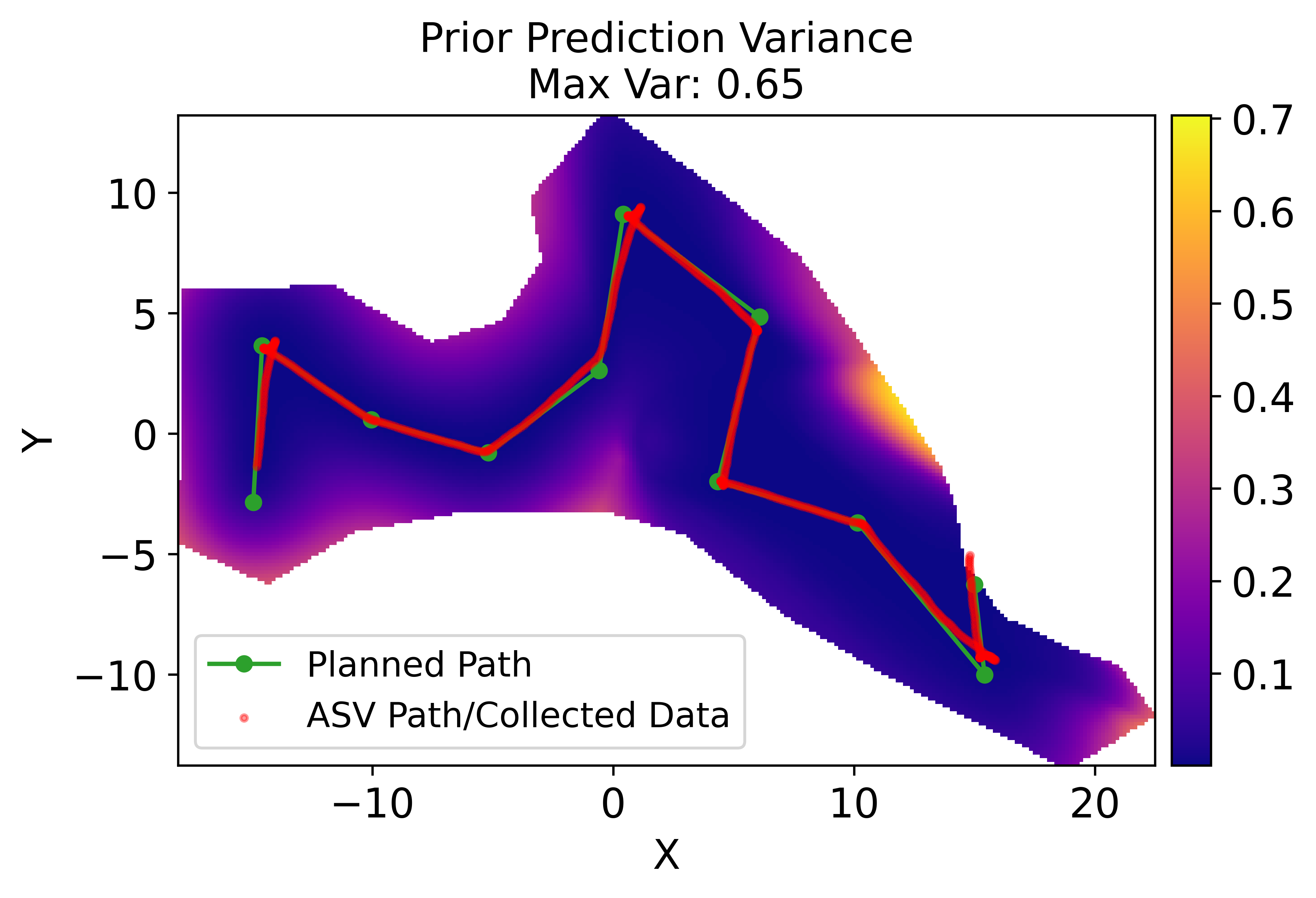}
    \end{subfigure}
    \hfill
    \begin{subfigure}{0.32\textwidth}
        \includegraphics[width=\textwidth]{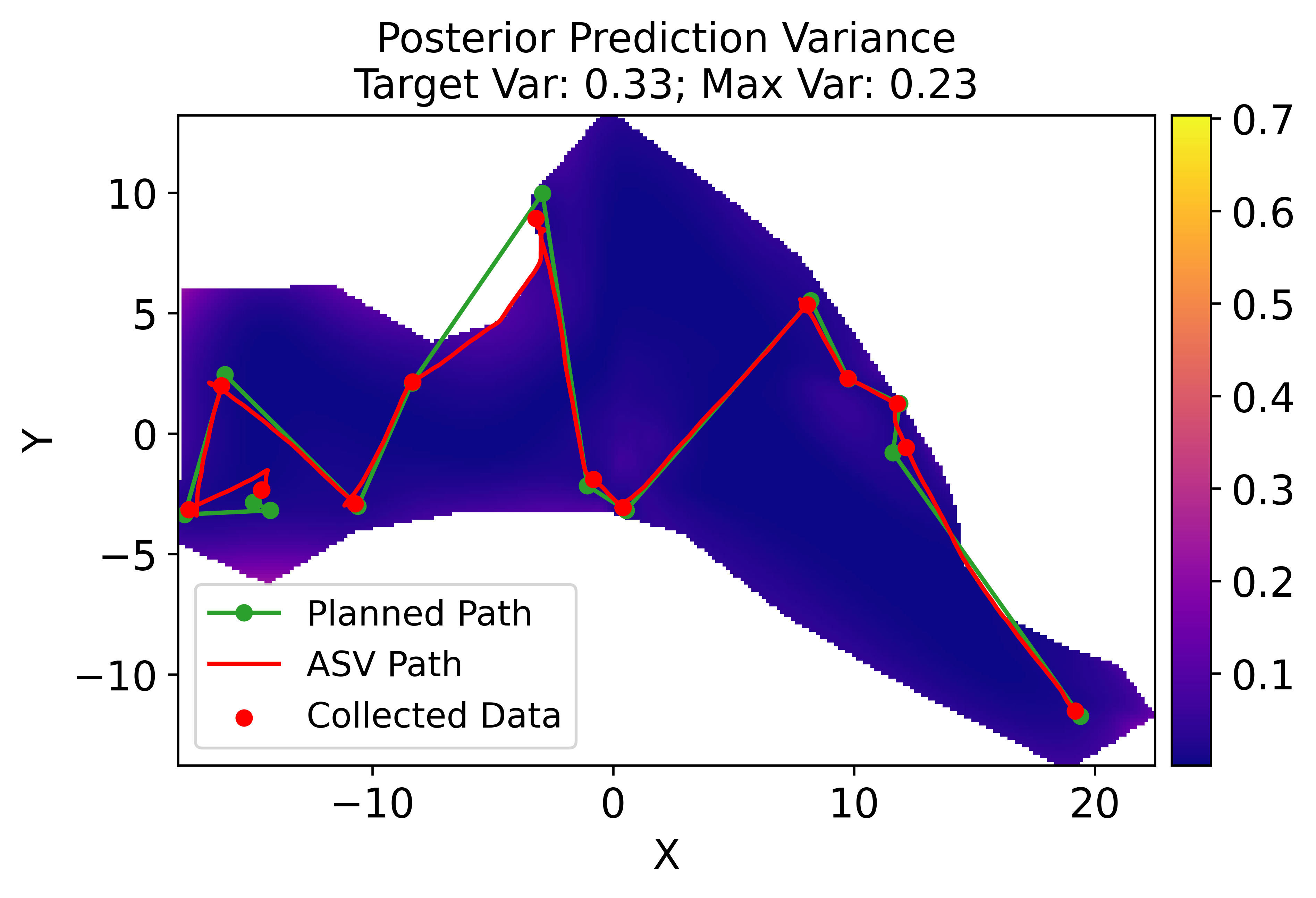}
    \end{subfigure}
    \caption{Autonomous surface vehicle (ASV) field-trial results using \textsc{GreedyCover}. Left: learned lengthscales from the non-stationary kernel. Middle: GP prior prediction variance conditioned on the initial pilot data. Right: GP posterior prediction variance after conditioning on the IPP solution data.}
    \label{fig:field_trial_global}
    \vspace{-0.1mm}
\end{figure*}

Computational efficiency varies significantly across the evaluated methods. \textsc{GreedyCover} is consistently the fastest, outperforming both \textsc{HexCover} and \textsc{Continuous-SGP}. For \textsc{HexCover}, while sensing-location selection requires less than one second, the overall runtime is dominated by the TSP solver. This overhead is driven by \textsc{HexCover} selecting a larger number of sensing locations than \textsc{GreedyCover}, which increases the complexity of the subsequent routing. In contrast, \textsc{GCBCover} is a little slower due to the tight coupling between its selection and routing phases. However, this integrated approach allows \textsc{GCBCover} to enforce strict distance constraints while maintaining near-optimal approximation guarantees for the joint selection and routing problem.

From a deployment perspective, the critical metrics are the number of sensing locations and the total path length. \textsc{GreedyCover} and \textsc{GCBCover} consistently satisfy the target variance threshold using fewer sensing locations and shorter travel distances. Although \textsc{HexCover} yields the lowest maximum posterior variance and a corresponding improvement in MSE, this advantage stems primarily from collecting a larger volume of data due to its inability to leverage non-stationary correlations. Conversely, \textsc{GreedyCover} and \textsc{GCBCover} utilize the Attentive kernel to model non-stationarity in the underlying field, effectively capturing spatially varying correlations to prioritize the most informative sensing locations.

\subsection{Autonomous Surface Vehicle Field Trial}
\label{exp:field_trial}

We validated the proposed method in a real-world deployment using a custom-built autonomous surface vehicle~(ASV)~\cite{Brancato22} to map lake bathymetry. The monitoring region was restricted to a non-convex subset of the lake. The ASV platform (Figure~\ref{fig:cover}) was equipped with GPS, a single-beam sonar, and a Raspberry Pi~5 for onboard computation.

We conducted an IPP trial using \textsc{GreedyCover}. We first generated a pilot path with ten waypoints via $k$-means and collected 544 depth measurements along this path. Using this data, we fit a GP with an Attentive kernel~\cite{ChenKL22}. We set the target maximum posterior variance to (50\%) of the maximum variance obtained after executing the pilot path, and then computed the IPP solution. To accelerate convergence and reduce path length, we warm-started \textsc{GreedyCover} with the pilot data. Consequently, the planned IPP solution avoided revisiting locations near the pilot path and prioritized regions whose uncertainty remained above the target threshold. After executing the IPP solution, we updated the GP using measurements collected \emph{at the visited waypoints} and evaluated the resulting maximum posterior variance. Additional implementation details are provided in the code Appendix. The learned lengthscales, together with the planned and executed pilot and informative paths, are shown in Figure~\ref{fig:field_trial_global}.

Despite deviations from the planned path, the method successfully reduced uncertainty such that the maximum-variance threshold was satisfied over the monitoring region. We also observed that the IPP sensing locations remained within the region boundary, and the low-level trajectory planner~\cite{Vinther15} adjusted the executed path to remain inside the non-convex domain, consistent with the non-convex domain constraints. Details of the ocean field trial with an Aqua2 AUV~\cite{DudekGPSSTJGHRZMLZBG07, MeriauxWWWSGJCSGROSD26} are provided in the Appendix.

\section{Limitations}

A central limitation is the reliance on a well-specified GP kernel: if the kernel is misspecified, the bound on predictive variance may fail to be a strict lower bound on the true mean-squared prediction error. An important extension is \emph{online} IPP that adapts a non-stationary kernel from streaming data (e.g., POAM~\cite{ChenLK24}) and provides guarantees that are robust to hyperparameter uncertainty. Alternatively, correlations could be inferred from physics-based priors or historical data, reducing—or potentially eliminating—the need for pilot data collection and onboard kernel learning.

Second, our guarantees are enforced on a finite evaluation set \(\mathcal{V}\) and therefore may not hold uniformly over the continuous domain. Future work could exploit regularity of the GP variance (e.g., continuity or Lipschitz properties) to guide grid-resolution selection and lift guarantees from \(\mathcal{V}\) to the continuous domain.

Finally, the field trials highlight execution uncertainty (e.g., localization error and imperfect tracking). Incorporating these effects via chance-constrained formulations would further improve real-world reliability.

%% file: conclusion.tex
\section{Conclusion}

We presented two planners for \emph{informative path planning~(IPP) with guaranteed estimation uncertainty} that compute near-shortest (or budget-feasible) sensing paths whose measurements ensure that the Gaussian process~(GP) posterior variance at every evaluation point falls below a user-specified threshold. This provides coverage-style guarantees while leveraging correlation-aware sampling.

To make the problem tractable, we fit a GP to available prior information and construct \emph{binary coverage maps} by
thresholding kernel-induced covariance, reducing the global variance requirement to a discrete coverage problem over
evaluation points. This reduction enables efficient planning with guarantees: \textsc{GreedyCover} yields a near-optimal
approximation for selecting sensing locations, while \textsc{GCBCover} balances marginal coverage gain against marginal
path length to obtain near-optimal solutions under a strict distance budget. Unlike geometric sweep patterns, both
methods adapt sampling to the learned correlation structure.

Our approach leverages \emph{non-stationary} kernels to adaptively sample complex fields, concentrating efforts in high-variation regions while avoiding redundancy in predictable areas. The approach naturally extends to \emph{non-convex} domains containing obstacles. Benchmarks on SRTM topography confirm that \textsc{GreedyCover} and \textsc{GCBCover} outperform existing IPP baselines—including those with and without formal uncertainty guarantees—by satisfying uncertainty constraints while minimizing both number of sensing locations and travel distance. Finally, field trials using autonomous surface and underwater vehicles for bathymetry mapping validate the real-world feasibility and robustness of our methods.

Future work includes IPP with guaranteed estimation uncertainty that accommodates vehicle dynamics, motion uncertainty, continuous sensing, and decentralized multi-robot deployments that maintain robustness in large-scale, communication-limited settings.

%% file: appendix.tex
\clearpage
\appendix
\setcounter{theorem}{0} 

\subsection{Theorems and Proofs}
\label{app:theorems_proofs}

This appendix provides the theoretical foundations and technical derivations that underpin our proposed approach.

\subsubsection{Coverage-map construction}
The binary coverage maps used by \textsc{GreedyCover} and \textsc{GCBCover}
are constructed by evaluating whether a \emph{single} candidate sensing location
$c$ can reduce the posterior variance at an evaluation point $v$ below the
target threshold $\sigma_{\mathrm{tar}}^2$. Theorem~\ref{thm:single_point}
gives an explicit covariance threshold for this single-location condition.
Theorem~\ref{thm:multi_point} shows that this condition is
\emph{conservative but safe}: conditioning on additional sensing locations
cannot increase posterior variance, so any point covered under a single-location remains covered under multi-location conditioning.

\begin{theorem}[Minimum Required Prior Covariance]
$\\$
Let $f \sim \mathcal{GP}(0,k)$ with independent Gaussian observation noise variance $\sigma_n^2$, and consider a single noisy observation at candidate location $c \in \mathcal{C}$. 
Fix an evaluation location $v \in \mathcal{V}$. 
For any target posterior variance $\sigma^2_{\mathrm{tar}} \in (0,k(v,v))$, the condition
\[
\sigma_{\mathrm{post}}^2(v \mid c) \le \sigma^2_{\mathrm{tar}}
\]
holds \emph{if and only if}
\begin{equation}
\label{eq:cov_threshold_detailed_apx}
|k(c,v)|
\;\ge\;
\sqrt{\bigl(k(v,v) - \sigma^2_{\mathrm{tar}}\bigr)\bigl(k(c,c) + \sigma_n^2\bigr)}.
\end{equation}
Thus, achieving a posterior variance of at most $\sigma^2_{\mathrm{tar}}$ at $v$ requires that the prior covariance between $c$ and $v$
exceed the threshold in~Equation~\ref{eq:cov_threshold_detailed_apx}.
\end{theorem}

\begin{proof}~Define 
$k_{cc} = k(c,c)$, $k_{cv} = k(c,v)$, and $k_{vv} = k(v,v)$.
The pair $(f(c),f(v))$ is jointly Gaussian with covariance
\[
\Sigma
=
\begin{bmatrix}
k_{cc} & k_{cv} \\
k_{cv} & k_{vv}
\end{bmatrix}.
\]
With observation noise variance $\sigma_n^2$, the Gaussian conditioning identity gives
\begin{align}
\label{eq:post_var_single_detailed}
\sigma_{\mathrm{post}}^2(v \mid c)
&=
k_{vv}
-
\frac{k_{cv}^2}{k_{cc} + \sigma_n^2}.
\end{align}
Requiring $\sigma_{\mathrm{post}}^2(v \mid c) \le \sigma^2_{\mathrm{tar}}$ is equivalent to
\[
k_{cv}^2 \ge \bigl(k_{vv} - \sigma^2_{\mathrm{tar}}\bigr)\bigl(k_{cc} + \sigma_n^2\bigr).\] 
Taking square roots yields~Equation~\ref{eq:cov_threshold_detailed_apx}.
\end{proof}

\begin{theorem}[Variance Reduction with Multiple Locations]
$\\$
Under the setup of Theorem~\ref{thm:single_point}, consider $n$ sensing locations $c_1,\dots,c_n \in \mathcal{C}$.
Fix an evaluation location $v \in \mathcal{V}$ and a target variance $\sigma^2_{\mathrm{tar}} \in (0,k(v,v))$. Suppose there exists
$j \in \{1,\dots,n\}$ such that
\[
|k(c_j,v)|
\;\ge\;
\sqrt{\bigl(k(v,v) - \sigma^2_{\mathrm{tar}}\bigr)\bigl(k(c_j,c_j) + \sigma_n^2\bigr)}.
\]
Then the posterior variance at $v$ conditioned on \emph{all} observations satisfies
\[
\sigma_{\mathrm{post}}^2\!\left(v \mid c_{1:n}\right) \le \sigma^2_{\mathrm{tar}}.
\]
\end{theorem}

\begin{proof}~By assumption, there exists $c_j$ such that $\sigma_{\mathrm{post}}^2(v \mid c_j) \le \sigma^2_{\mathrm{tar}}$.
For jointly Gaussian variables, conditional variance is monotone non-increasing under additional conditioning: if $(U,V,W)$ are jointly
Gaussian, then
\[
\sigma_{\mathrm{post}}^2(U \mid V,W) \le \sigma_{\mathrm{post}}^2(U \mid V),
\]
which follows from the fact that conditioning reduces entropy~\cite[Theorem~2.6.5]{Cover91}.
Applying this property yields
\[
\sigma_{\mathrm{post}}^2\!\left(v \mid c_{1:n}\right)
\le
\sigma_{\mathrm{post}}^2\!\left(v \mid c_j\right)
\le
\sigma^2_{\mathrm{tar}}.
\]
\end{proof}

\subsection{Algorithms: Coverage-Based Selection and Path Planning}
\label{app:algorithms}

This appendix details the pseudocode for the two proposed path planners. To reproduce our experimental results, the accompanying source code can be accessed at: \href{https://github.com/itskalvik/uncertainty-guaranteed-ipp}{https://github.com/itskalvik/uncertainty-guaranteed-ipp}

Both \textsc{GreedyCover} and \textsc{GCBCover} operate on a binary coverage matrix $B\in\{0,1\}^{M\times N}$
constructed from the learned GP kernel via the covariance threshold in
Theorem~\ref{thm:single_point}. The matrix entry $B_{ji}=1$ indicates
that sensing at candidate location $c_j$ is sufficient (under the single-location
test) to reduce posterior variance at evaluation point $v_i$ below
$\sigma^2_{\mathrm{tar}}$. 

We use both index-set and location-set notation interchangeably. For example,
a sensing set may be written as an index set $S \subseteq \{1,\dots,M\}$ or as
the corresponding location set $\{c_j : j \in S\} \subseteq \mathcal{C}$.
The intended meaning is clear from context.

\subsubsection{\textsc{GreedyCover}}
Algorithm~\ref{alg:greedy-coverage} greedily selects sensing locations to
maximize the monotone submodular coverage objective $F(S)$ (Equation~\ref{eq:coverage_objective})
by repeatedly choosing the candidate that covers the largest number of currently
uncovered evaluation points.

\begin{algorithm}[h]
\caption{\textsc{GreedyCover}: Estimation Uncertainty Guaranteed Sensor Placement}
\label{alg:greedy-coverage}
\begin{algorithmic}[1]
\Require Coverage matrix $B \in \{0,1\}^{M \times N}$, evaluation set $\mathcal{V}=\{v_1,\dots,v_N\}$.
\Ensure Selected index set $S \subseteq \{1,\dots,M\}$.
\State $S \gets \emptyset$ \Comment{selected sensing-location indices}
\State $U \gets \emptyset$ \Comment{covered evaluation points}
\While{$|U| < N$}
    \For{each candidate $j \in \{1,\dots,M\}\setminus S$}
        \State $g(j) \gets \bigl|\{v_i \in \mathcal{V} : B_{ji}=1,\ v_i \notin U\}\bigr|$ 
        \State \Comment{marginal gain}
    \EndFor
    \State $j^\star \gets \arg\max_{j} g(j)$
    \If{$g(j^\star) = 0$}
        \State \textbf{break} \Comment{no further improvement possible}
    \EndIf
    \State $S \gets S \cup \{j^\star\}$
    \State $U \gets U \cup \{v_i \in \mathcal{V} : B_{j^\star i}=1\}$ \label{line:greedy_coverage_update}
\EndWhile
\State \Return $S$
\end{algorithmic}
\end{algorithm}

\subsubsection{\textsc{GCBCover}}
Algorithm~\ref{alg:gcb} extends \textsc{GreedyCover} to the routing-constrained
setting by selecting candidates based on a coverage-per-cost ratio.
It also computes a (truncated) greedy-coverage baseline and returns the better
of the two, consistent with the GCB algorithm.

\begin{algorithm}[ht]
\caption{\textsc{GCBCover}: Estimation Uncertainty Guaranteed Informative Path Planning}
\label{alg:gcb}
\begin{algorithmic}[1]
\Require Coverage matrix $B\in\{0,1\}^{M\times N}$, cost  budget $D$.
\Ensure Selected ordered index set $S \subseteq \{1,\dots,M\}$.
\State $S_{\mathrm{Greedy}} \gets \textsc{GreedyCover}(B)$ \Comment{Algorithm~\ref{alg:greedy-coverage}}
\State Compute a route $\mathcal{P}_{\mathrm{Greedy}}$ through $\{c_j:j\in S_{\mathrm{Greedy}}\}$ and truncate it to cost $\le D$
\State Let $S_{\mathrm{Greedy}}$ be the indices visited by the truncated route
\State $S \gets \emptyset$; \ $U \gets \emptyset$ \Comment{current selection and covered set}
\While{$\mathrm{cost}(S) \le D$ \textbf{and} $|U|<N$}
  \ForAll{$j \in \{1,\dots,M\}\setminus S$}
    \State $\Delta F_j \gets F(S \cup \{j\}) - F(S)$
    \State $\Delta \mathrm{cost}_j \gets \mathrm{cost}(S \cup \{j\}) - \mathrm{cost}(S)$
    \State $r(j)\gets \Delta F_j/\Delta \mathrm{cost}_j$
  \EndFor
  \State $j^\star \gets \arg\max_j r(j)$
  \If{$\mathrm{cost}(S \cup \{j^\star\}) \le D$}
    \State $S\gets S \cup \{j^\star\}$;\; $U\gets U\cup\{v_i: B_{j^\star i}=1\}$
  \EndIf
  \State Remove $j^\star$ from future consideration
\EndWhile
\State \Return  $\arg\max_{A\in \{S_{\mathrm{Greedy}},\, S\}} F(A)$
\end{algorithmic}
\end{algorithm}

\subsubsection{Routing-Cost Heuristic and Computational Efficiency}
\label{sec:routing_heuristic}

The GCB selection rule requires evaluating the marginal routing-cost increment
$\Delta \mathrm{cost}(j\mid S)$ for many candidate additions.
A naive implementation would repeatedly solve a routing problem (e.g., a TSP)
for each candidate at each iteration, which is computationally expensive.

To reduce the number of full routing solves, we estimate $\Delta \mathrm{cost}(j\mid S)$
using a fast nearest-insertion heuristic: we consider inserting the new location
$c_j$ along each edge of the current route and take the best (smallest) increase
in route cost. This yields an efficient approximation
$\widehat{\Delta \mathrm{cost}}(j\mid S)$ that can be computed quickly for all
candidates.

\begin{algorithm}[h]
\caption{Approximate route cost increment (nearest-insertion heuristic)}
\label{alg:approx-route-increment}
\begin{algorithmic}[1]
\Function{ApproxRouteIncrement}{$j, S, \mathcal{P}_S$}
    \State Let \(\mathcal{P}_S\) be the current route visiting \(\{c_k : k \in S\}\).
    \State Initialize \(\widehat{c} \gets \infty\).
    \For{each edge \((c_a,c_b)\) along \(\mathcal{P}_S\)}
        \State Consider inserting \(c_j\) between \(c_a\) and \(c_b\).
        \State \(\widehat{c}_{ab} \gets \mathrm{cost}(\mathcal{P}_S)
               + \|c_a - c_j\|\)
        \State \(\quad \quad \quad \quad \quad \quad \quad + \|c_j - c_b\|
               - \|c_a - c_b\|\).
        \State \(\widehat{c} \gets \min(\widehat{c}, \widehat{c}_{ab})\).
    \EndFor
    \State \Return \(\widehat{c} - \mathrm{cost}(\mathcal{P}_S)\).
\EndFunction
\end{algorithmic}
\end{algorithm}

\begin{figure*}[!ht]
    \centering
    \begin{subfigure}{0.453\textwidth}
        \includegraphics[width=\textwidth]{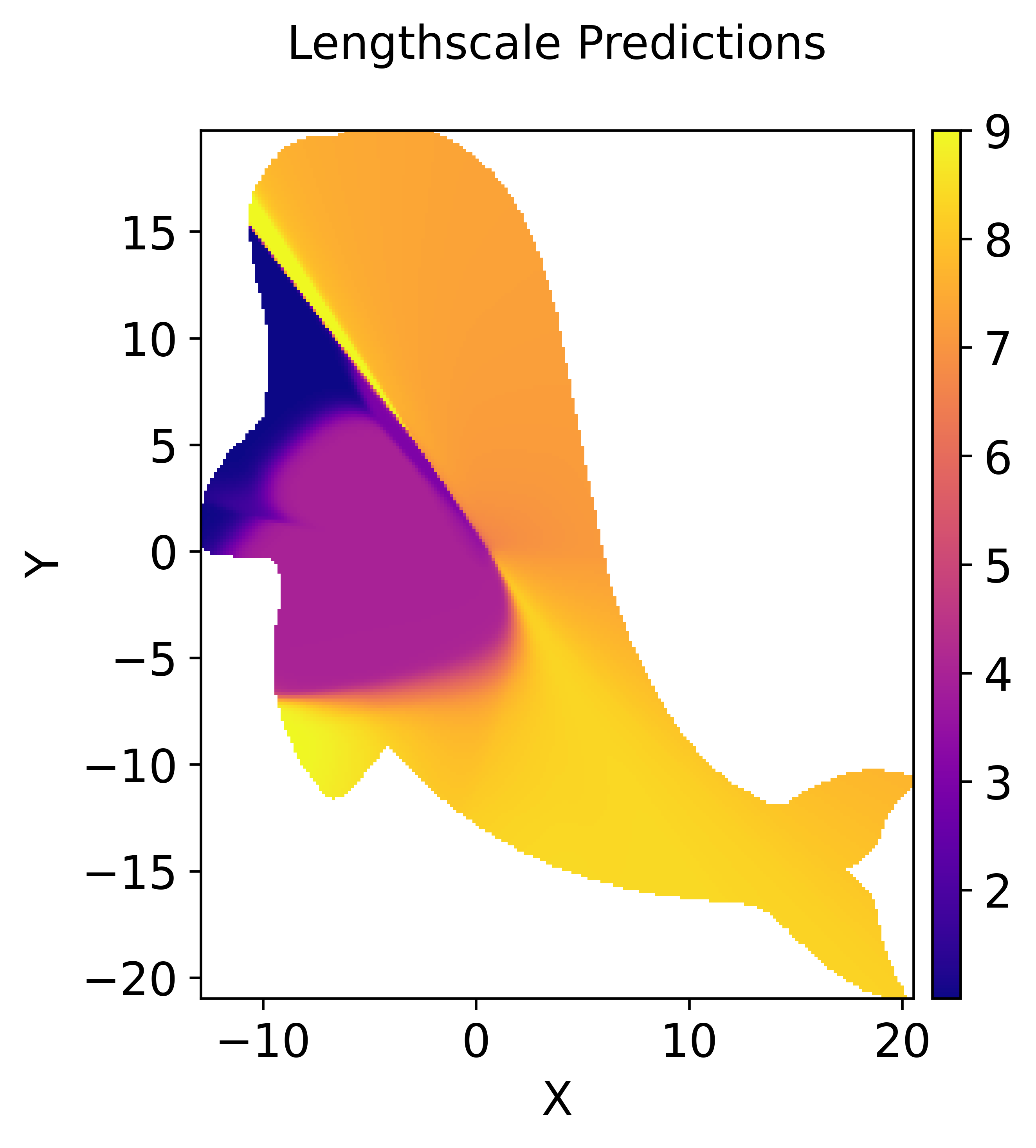}
    \end{subfigure}
    \hfill
    \begin{subfigure}{0.49\textwidth}
        \includegraphics[width=\textwidth]{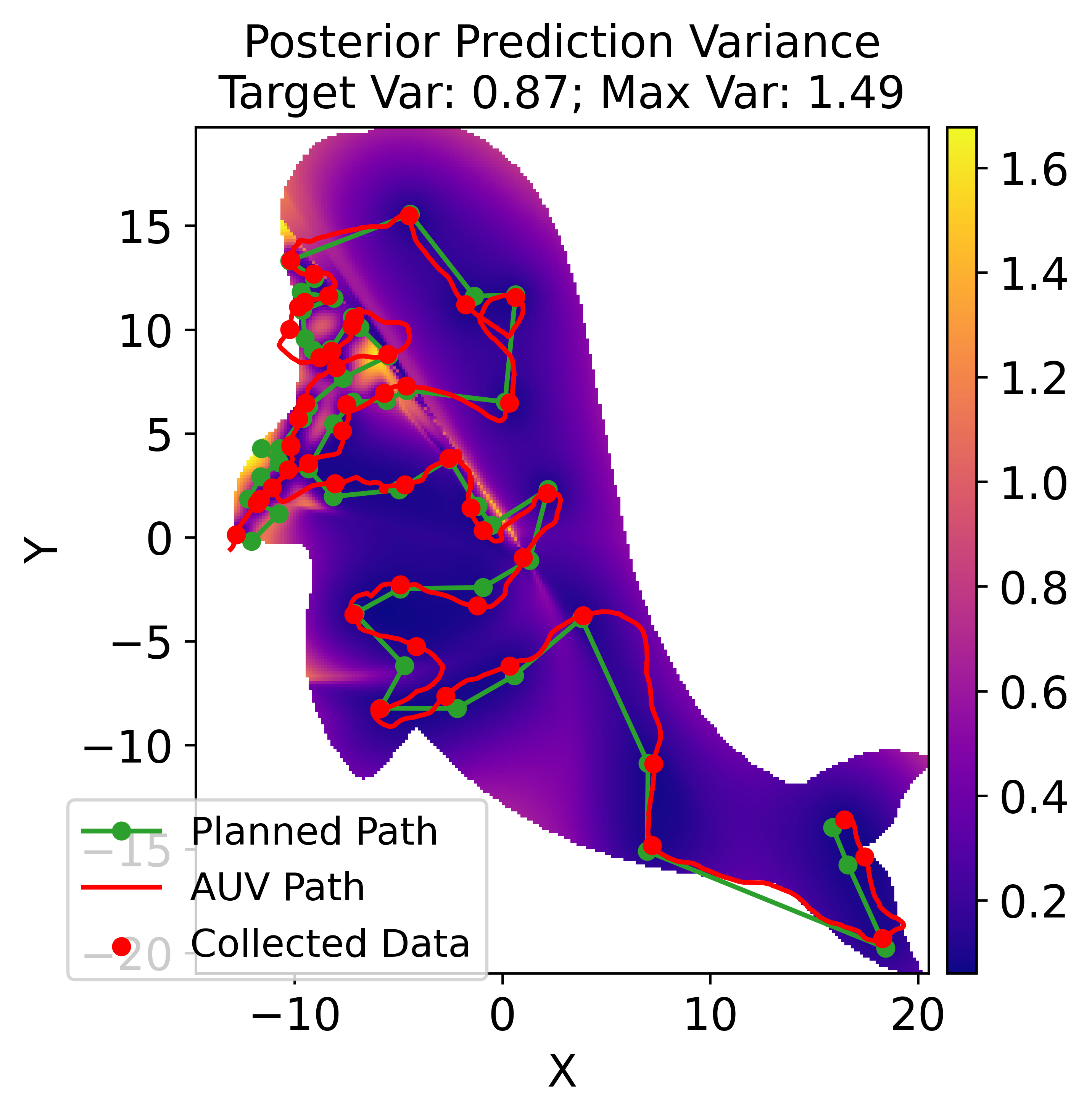}
    \end{subfigure}
    \caption{Autonomous underwater vehicle (AUV) field-trial results using \textsc{GreedyCover}. Left: learned lengthscales from the non-stationary kernel. Right: GP posterior predictive variance after conditioning on the IPP solution data.}
    \label{fig:field_trial_local}
\end{figure*}

\begin{figure*}[!ht]
    \centering
    \begin{subfigure}{0.32\textwidth}
        \includegraphics[width=\textwidth]{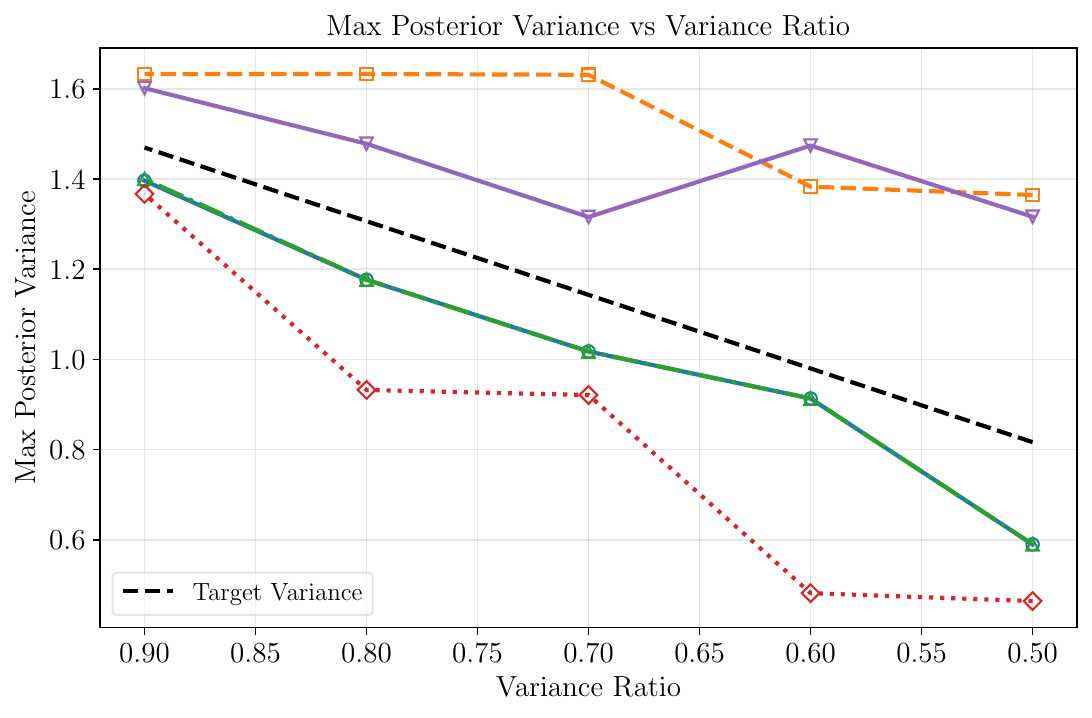}
    \end{subfigure}
    \hfill
    \begin{subfigure}{0.32\textwidth}
        \includegraphics[width=\textwidth]{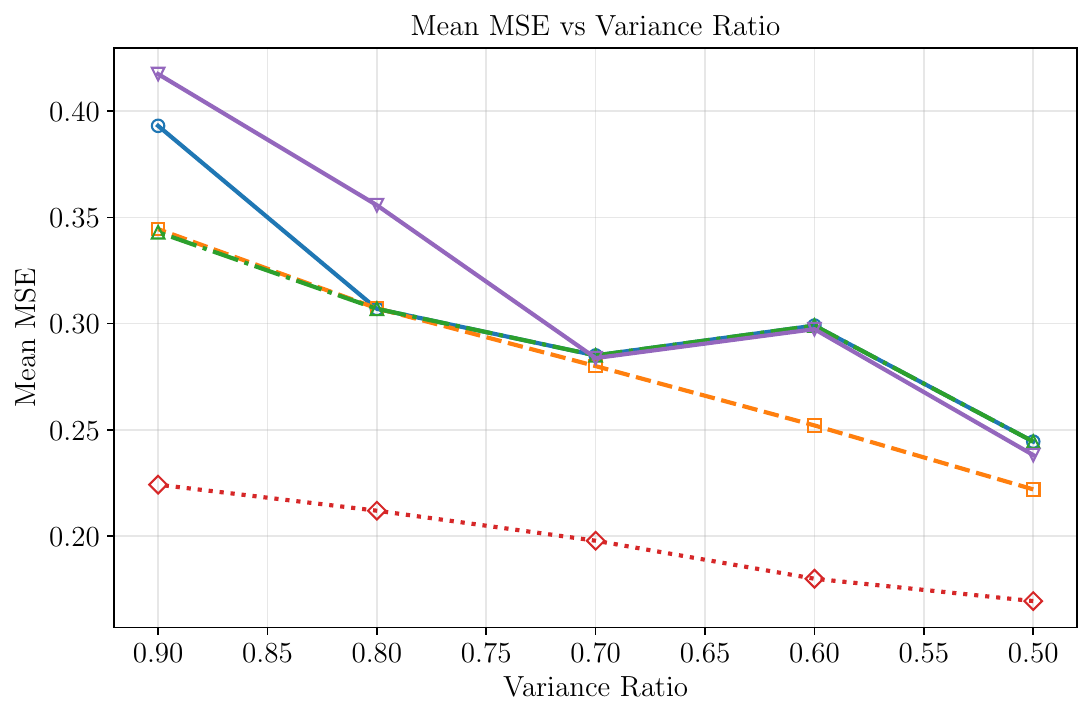}
    \end{subfigure}
    \hfill
    \begin{subfigure}{0.32\textwidth}
        \includegraphics[width=\textwidth]{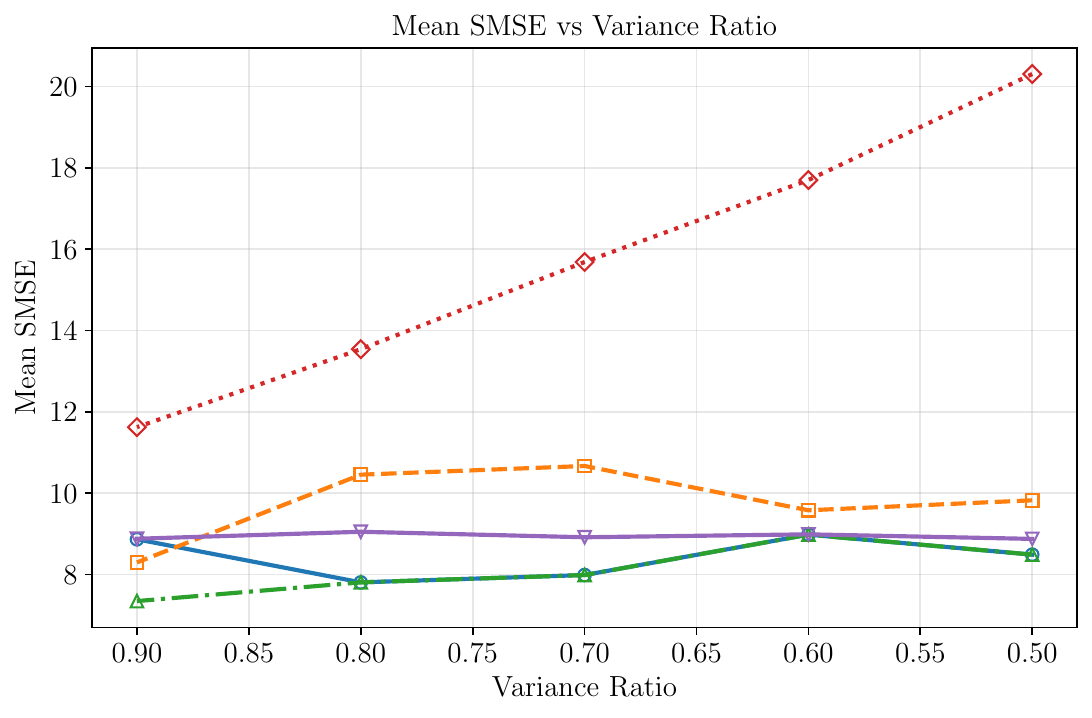}
    \end{subfigure}
    \hfill
    \begin{subfigure}{0.32\textwidth}
        \includegraphics[width=\textwidth]{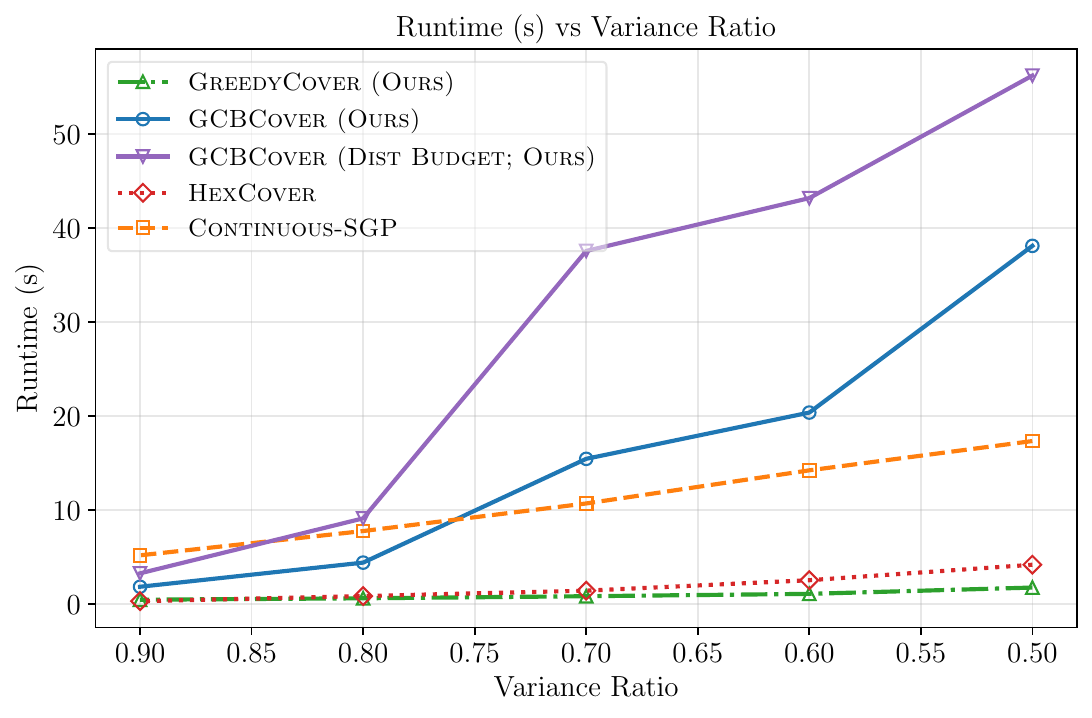}
    \end{subfigure}
    \hfill
    \begin{subfigure}{0.32\textwidth}
        \includegraphics[width=\textwidth]{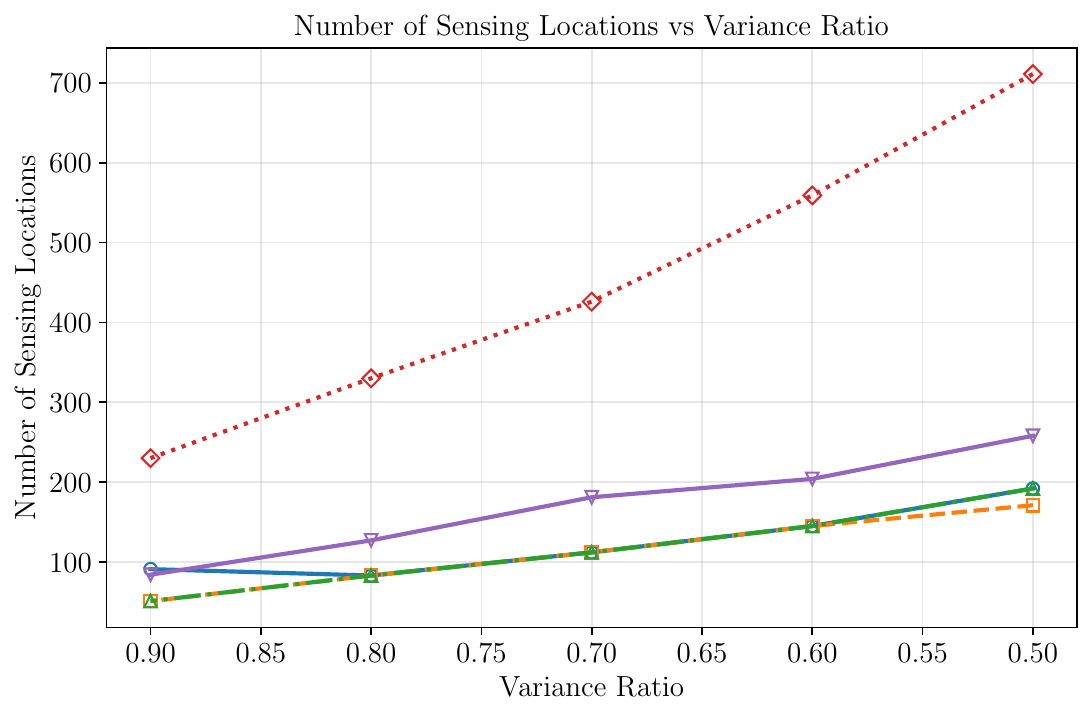}
    \end{subfigure}
    \hfill
    \begin{subfigure}{0.32\textwidth}
        \includegraphics[width=\textwidth]{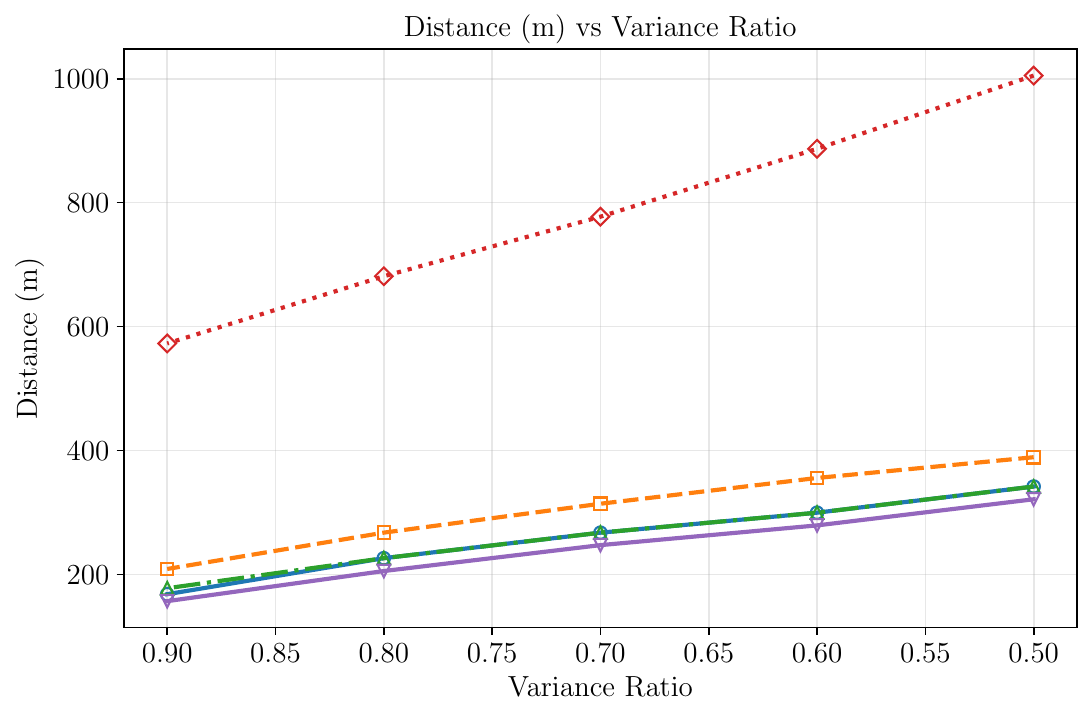}
    \end{subfigure}
    \caption{Benchmark results on the SRTM dataset \((47^\circ\mathrm{N},\,124^\circ\mathrm{W})\). Top row: Max posterior variance, MSE, and SMSE. Bottom row: Runtime, sensing locations, and path length. Lower values represent better performance across all metrics.}
    \label{fig:benchmark-N47W124}
\end{figure*}

In the full \textsc{GCBCover} procedure, we therefore:
(i) compute approximate increments for all candidates to rank them by the
coverage-per-cost ratio, and (ii) invoke the exact routing solver only for
the single top-ranked candidate to update the route and verify feasibility.
This substantially reduces runtime while preserving the intended selection
behavior (prioritizing high marginal coverage per unit cost).

\subsubsection{Joint Posterior Uncertainty}
It is important to distinguish the \emph{coverage objective} $F$ from the
underlying uncertainty-reduction mechanism used to construct the coverage sets.
In our coverage-map construction (Section~\ref{sec:coverage_map_construction}),
each candidate $j$ is associated with a fixed set of covered points computed in
isolation (i.e., by conditioning on a single candidate at a time). This induces
the union-of-sets objective in Equation~\ref{eq:coverage_objective}, which is
monotone submodular and therefore admits the near-optimal approximation guarantees.

When multiple sensing locations are chosen together, their joint impact on
posterior uncertainty can be \emph{super-additive}: some evaluation points fall
below the uncertainty threshold only after conditioning on several locations at
once. Because these interaction effects are not captured by per-candidate
coverage sets, the resulting estimate $F(S)$ is conservative. This conservatism
works in our favor: if the algorithm reports that a set $S$ meets a target
coverage level under the isolated-candidate model, then the \emph{true} maximum
posterior variance after executing $S$ is typically no greater—and often
strictly lower—than the target threshold.

A tighter (but more expensive) alternative is to recompute the true achieved coverage at each iteration by evaluating the GP posterior variance over the evaluation set under the current selection $S_t$, and then updating coverage accordingly (Algorithm~\ref{alg:greedy-coverage}, line~\ref{line:greedy_coverage_update}). This explicitly captures joint (multi-location) effects rather than relying on per-candidate coverage.

\subsection{Autonomous Underwater Vehicle Field Trial}

This appendix describes our ocean field trial conducted in the Folkestone Marine Reserve, Barbados, utilizing an Aqua2 autonomous underwater vehicle (AUV)~\cite{DudekGPSSTJGHRZMLZBG07, MeriauxWWWSGJCSGROSD26} (Figure~\ref{fig:cover}) for bathymetric mapping. The Aqua2 is a non-holonomic platform that employs proportional-derivative (PD) controllers to actuate its six-flipper propulsion system. The survey was restricted to a non-convex subset of the monitoring region. To facilitate localization and provide altitude data, the vehicle was equipped with a Doppler Velocity Logger (DVL). Throughout the deployment, the AUV was commanded to maintain a constant depth of 0.5~m below the surface.

We conducted an IPP trial using \textsc{GreedyCover}. A GP with an Attentive kernel~\cite{ChenKL22} was fit using preexisting data comprising 490 samples from the region. We set the target maximum posterior variance to $0.87$ and computed the IPP solution. The solution was generated without warm starting; this choice produced a longer path, enabling us to better assess the impact of navigation uncertainty. Planning was performed offboard on a surface laptop. After executing the planned path, we updated the GP using measurements collected \emph{at the visited waypoints} and evaluated the resulting maximum posterior variance. The learned lengthscales, together with the planned and executed informative paths, are shown in Figure~\ref{fig:field_trial_local}.

The vehicle remained close to the monitoring-region boundary, but we observed notable deviations from the planned path. We attribute these deviations to strong ocean currents and to the fact that the TSP solver did not account for vehicle motion constraints. Consequently, the execution did not achieve the target variance at every evaluation location within the region. Nonetheless, the collected data substantially reduced uncertainty in high-variance areas while minimizing travel distance by prioritizing locations with short lengthscales and high uncertainty.

\subsection{Additional Benchmark Results}
\label{sec:additional_results}

This appendix presents comprehensive benchmarking results, detailed in Figures~\ref{fig:benchmark-N47W124} through~\ref{fig:benchmark-N17E073}.

\begin{figure*}[!ht]
   \centering
    \begin{subfigure}{0.32\textwidth}
        \includegraphics[width=\textwidth]{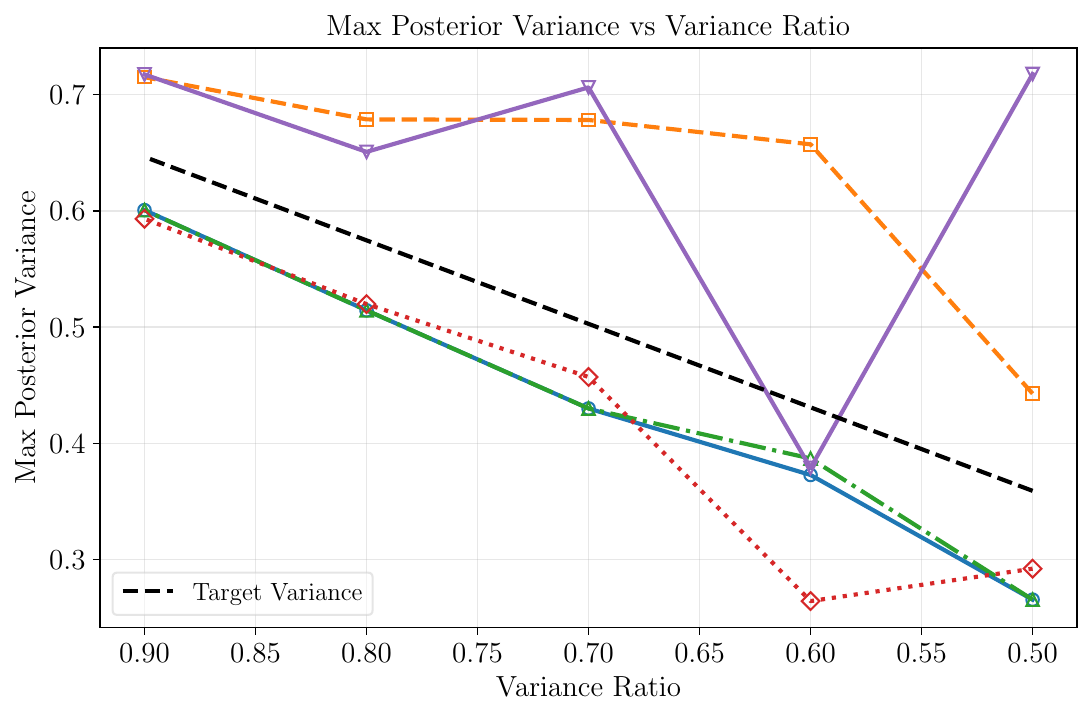}
    \end{subfigure}
    \hfill
    \begin{subfigure}{0.32\textwidth}
        \includegraphics[width=\textwidth]{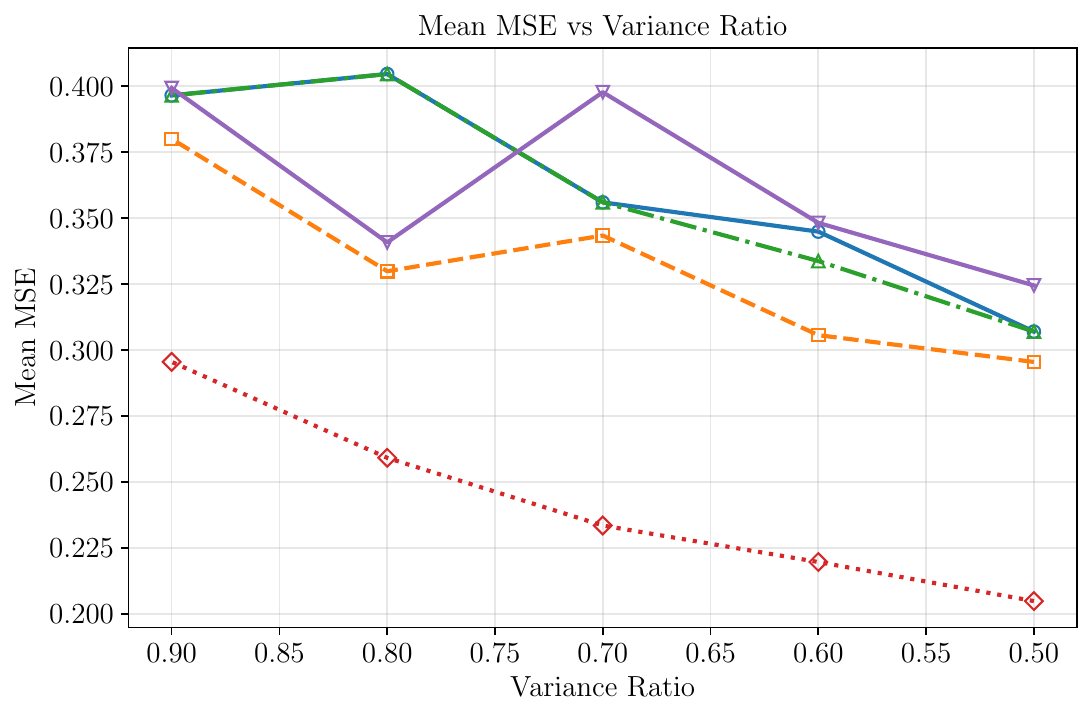}
    \end{subfigure}
    \hfill
    \begin{subfigure}{0.32\textwidth}
        \includegraphics[width=\textwidth]{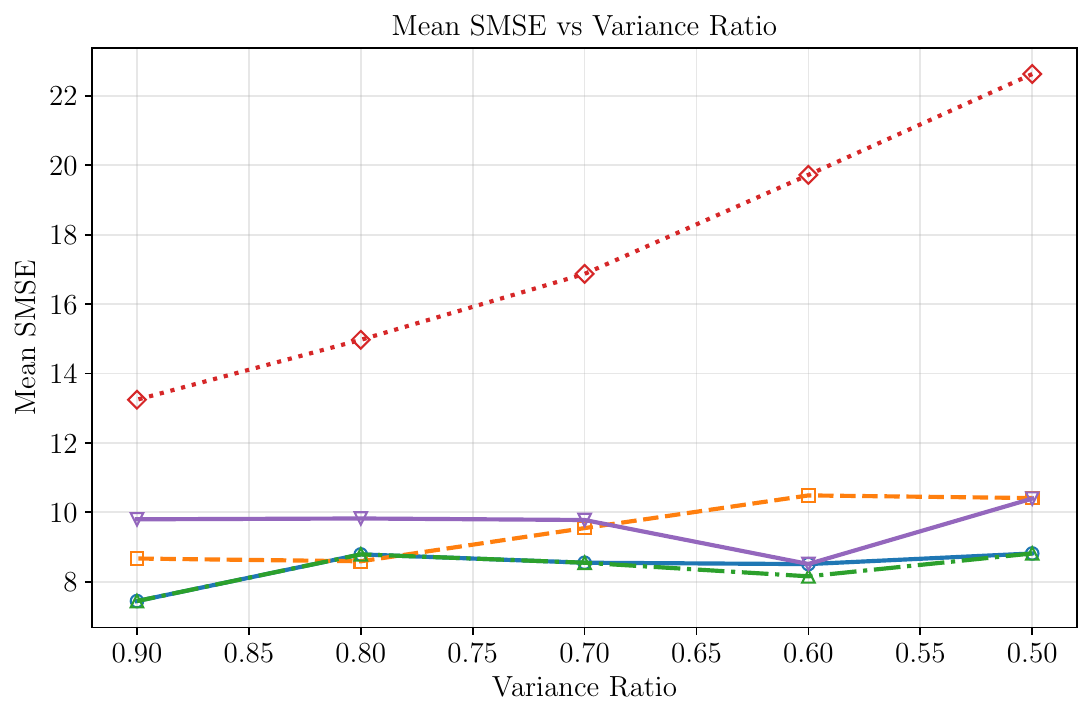}
    \end{subfigure}
    \hfill
    \begin{subfigure}{0.32\textwidth}
        \includegraphics[width=\textwidth]{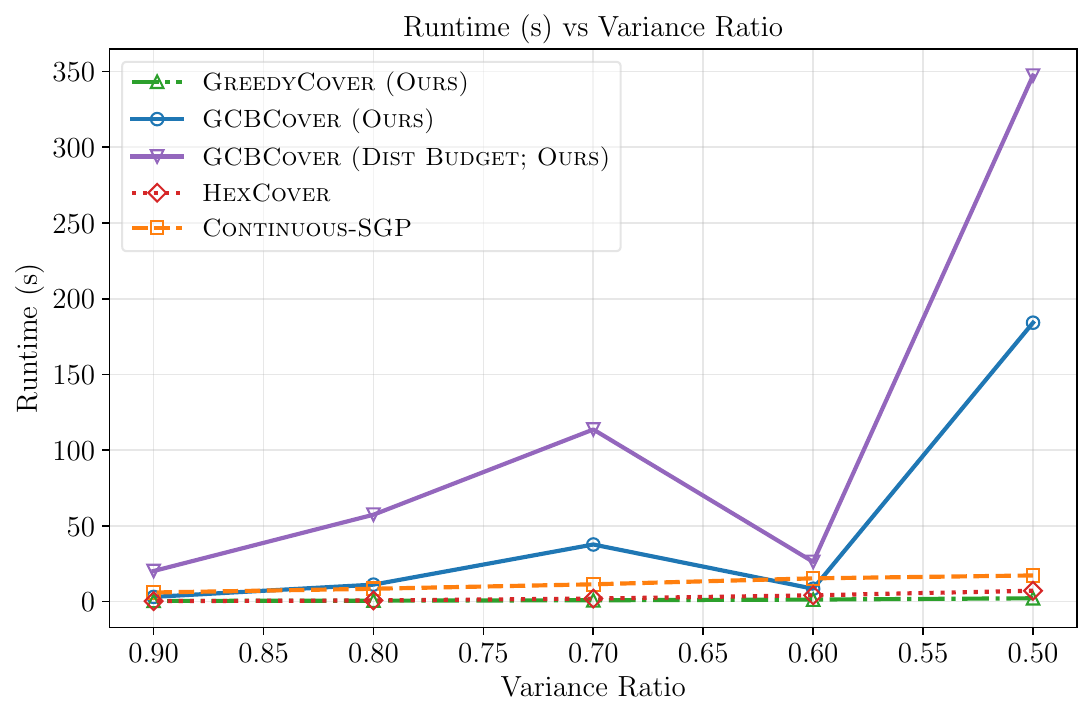}
    \end{subfigure}
    \hfill
    \begin{subfigure}{0.32\textwidth}
        \includegraphics[width=\textwidth]{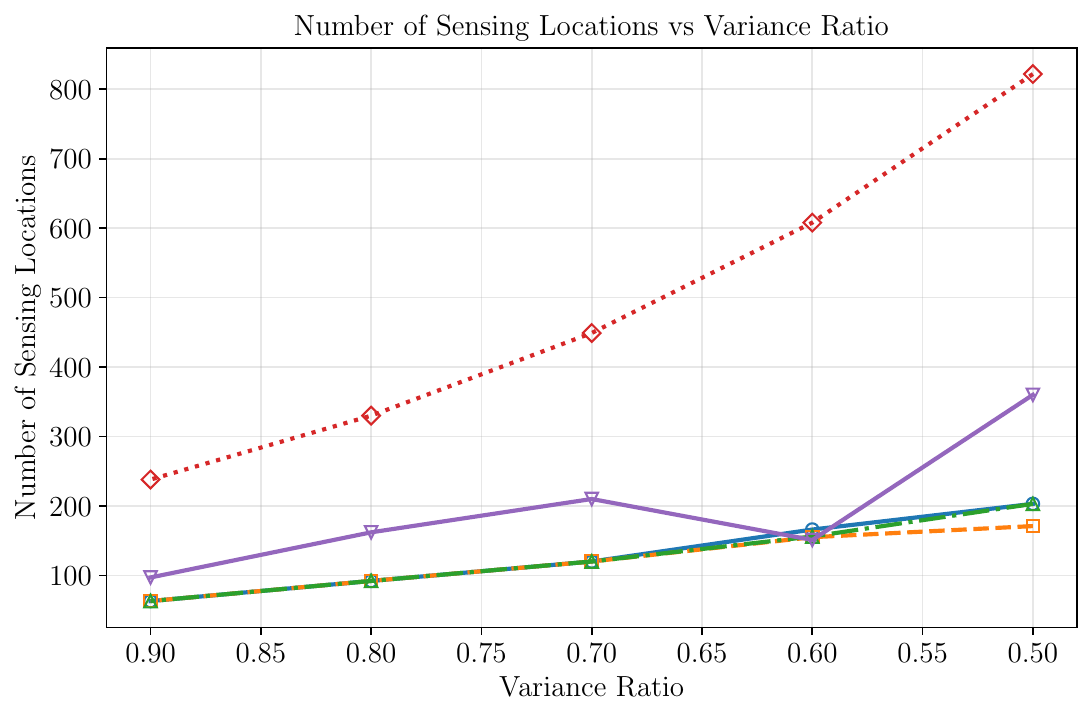}
    \end{subfigure}
    \hfill
    \begin{subfigure}{0.32\textwidth}
        \includegraphics[width=\textwidth]{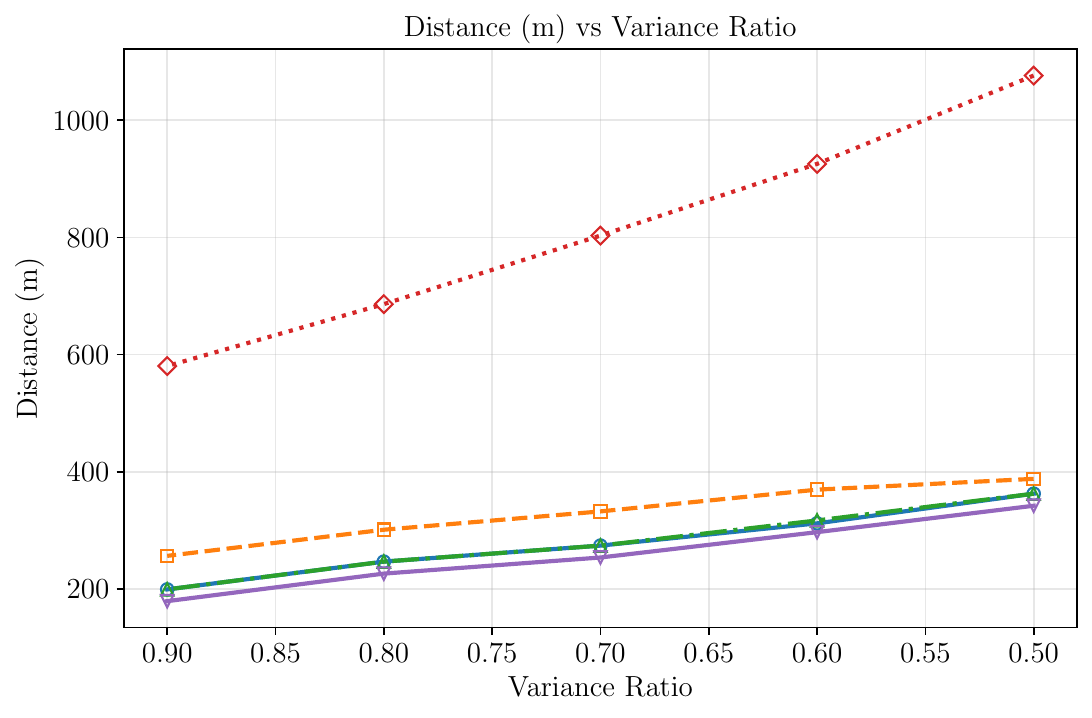}
    \end{subfigure}
    \caption{Benchmark results on the SRTM dataset \((02^\circ\mathrm{N},\,021^\circ\mathrm{W})\). Top row: Max posterior variance, MSE, and SMSE. Bottom row: Runtime, sensing locations, and path length. Lower values represent better performance across all metrics.}
    \label{fig:benchmark-N02E021}
\end{figure*}

\begin{figure*}[!ht]
   \centering
    \begin{subfigure}{0.32\textwidth}
        \includegraphics[width=\textwidth]{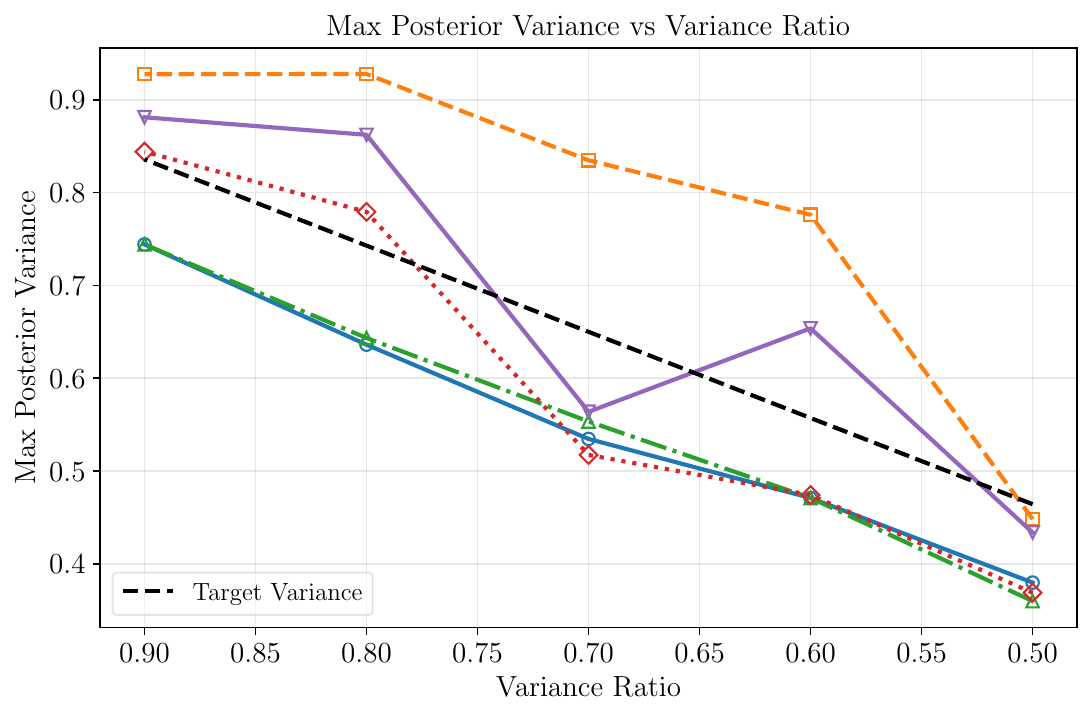}
    \end{subfigure}
    \hfill
    \begin{subfigure}{0.32\textwidth}
        \includegraphics[width=\textwidth]{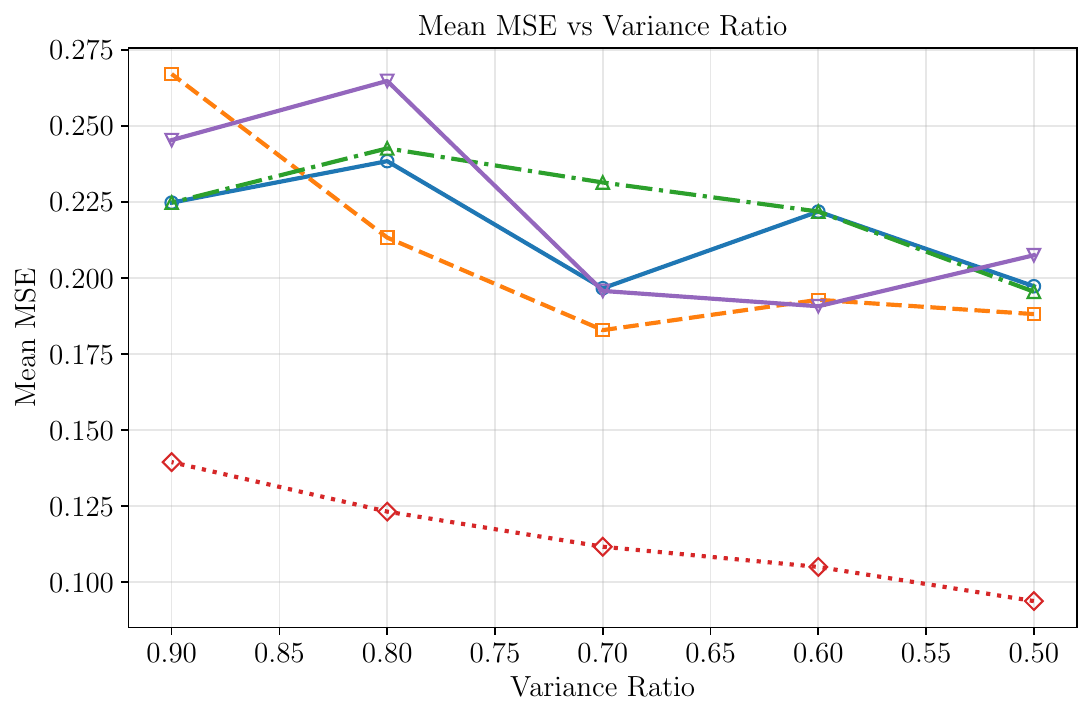}
    \end{subfigure}
    \hfill
    \begin{subfigure}{0.32\textwidth}
        \includegraphics[width=\textwidth]{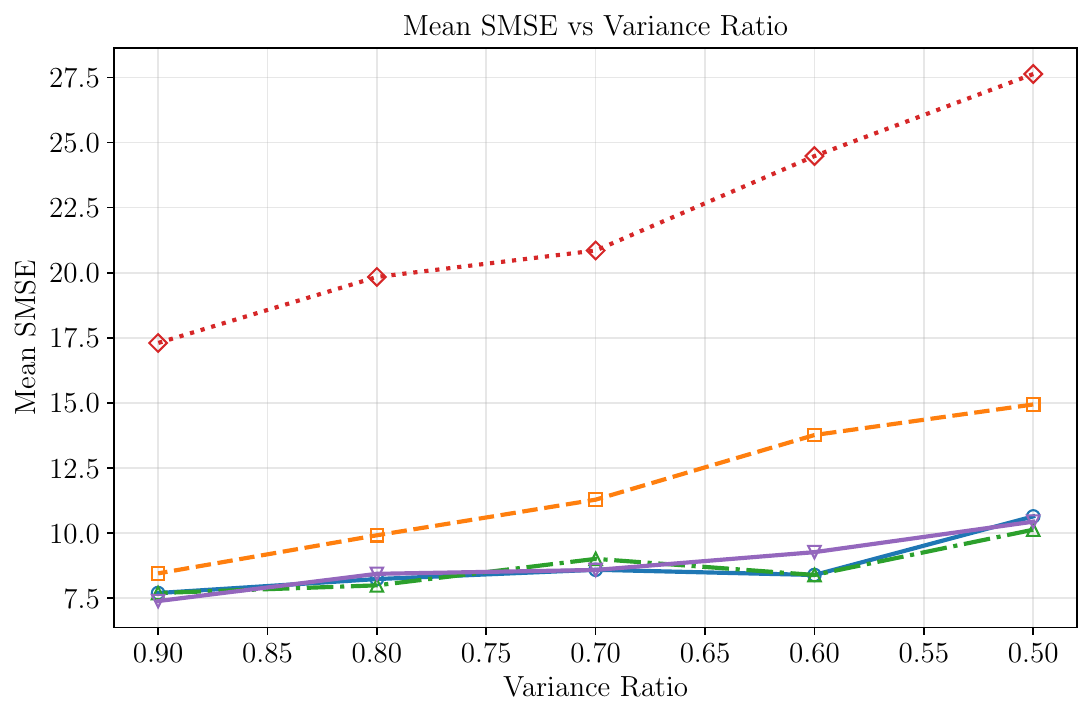}
    \end{subfigure}
    \hfill
    \begin{subfigure}{0.32\textwidth}
        \includegraphics[width=\textwidth]{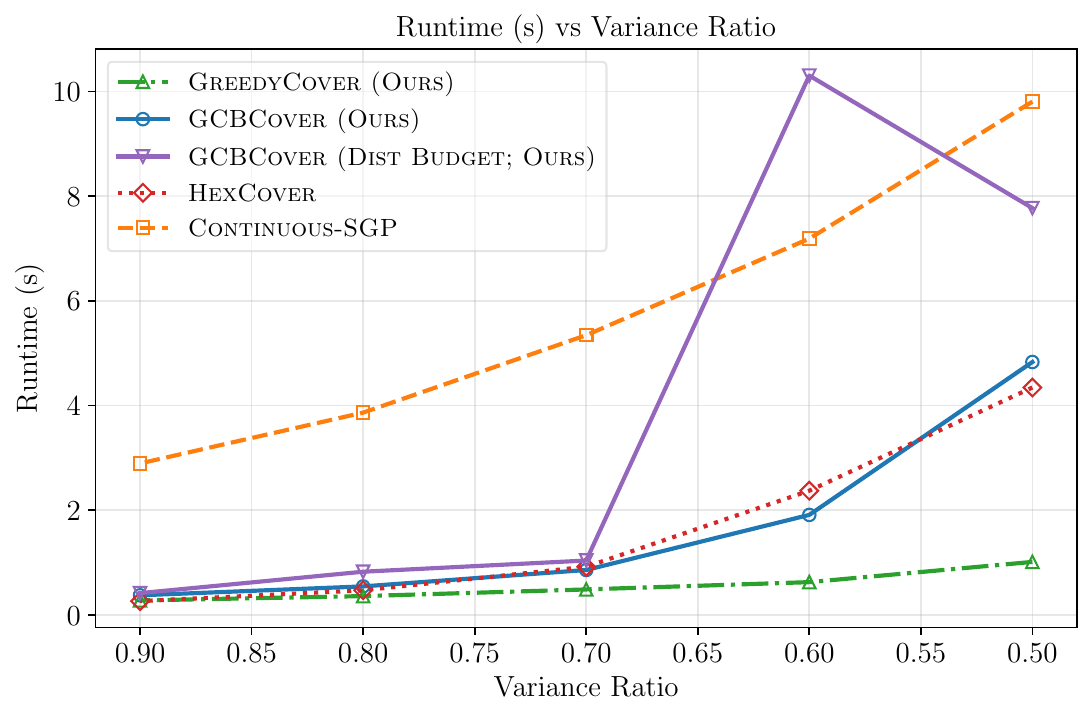}
    \end{subfigure}
    \hfill
    \begin{subfigure}{0.32\textwidth}
        \includegraphics[width=\textwidth]{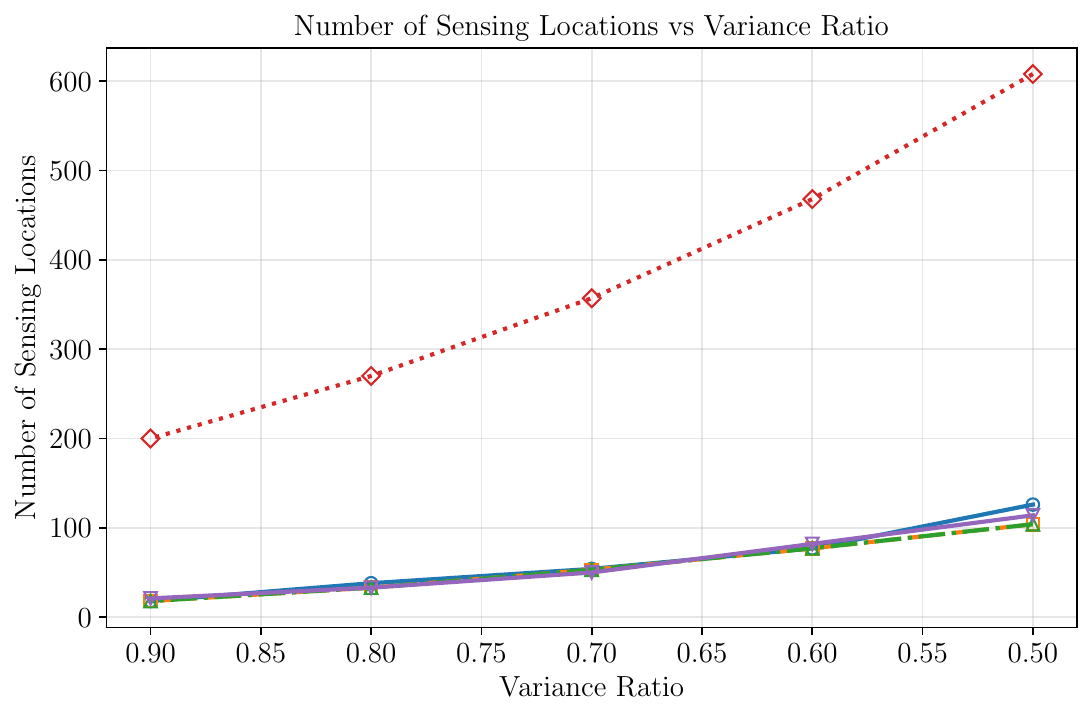}
    \end{subfigure}
    \hfill
    \begin{subfigure}{0.32\textwidth}
        \includegraphics[width=\textwidth]{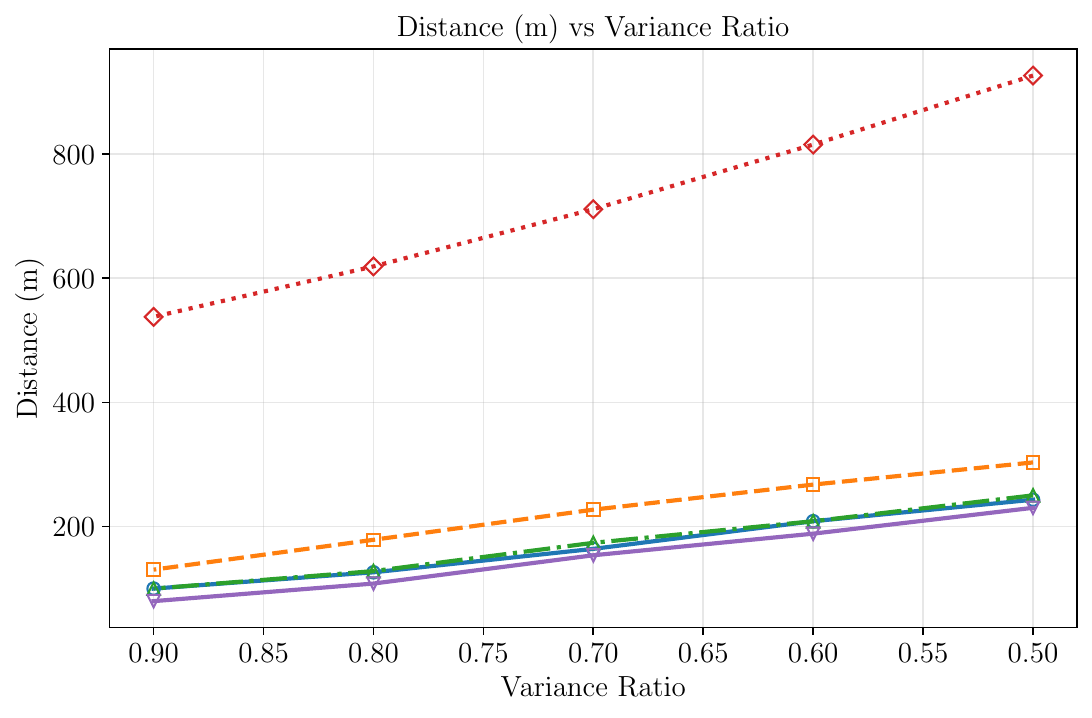}
    \end{subfigure}
    \caption{Benchmark results on the SRTM dataset \((17^\circ\mathrm{N},\,073^\circ\mathrm{W})\). Top row: Max posterior variance, MSE, and SMSE. Bottom row: Runtime, sensing locations, and path length. Lower values represent better performance across all metrics.}
    \label{fig:benchmark-N17E073}
\end{figure*}